\documentclass[letterpaper]{article} 
\usepackage[draft]{aaai2026}  
\usepackage{times}  
\usepackage{helvet}  
\usepackage{courier}  
\usepackage[hyphens]{url}  
\usepackage{graphicx} 
\usepackage{natbib}  
\usepackage{caption} 
\usepackage{algorithm}
\usepackage{newfloat}
\usepackage{listings}
\DeclareCaptionStyle{ruled}{labelfont=normalfont,labelsep=colon,strut=off} 
\floatstyle{ruled}
\newfloat{listing}{tb}{lst}{}
\floatname{listing}{Listing}
\usepackage{booktabs}
\usepackage{xspace}
\usepackage[noend]{algpseudocode}
\usepackage{amsmath}
\usepackage{amssymb}
\usepackage{amsthm}
\usepackage{mathtools}
\usepackage{multirow}
\usepackage{paralist}
\usepackage{comment}
\usepackage{subcaption}
\usepackage{makecell}
\usepackage{enumitem}

\usepackage{tikz,pgfplots,pgfmath}
\usepackage{xcolor}
\usepackage{fontawesome5}
\usetikzlibrary{shapes.misc}
\usetikzlibrary{positioning, arrows.meta, decorations.pathmorphing,shapes}

\newcommand{\tcore}{\textsc{{tcore}}\xspace}

\newcommand{\A}{\ensuremath{\mathsf{A}}}
\newcommand{\ST}{\ensuremath{\mathsf{ST}}}
\newcommand{\SB}{\ensuremath{\mathsf{SB}}}
\newcommand{\SA}{\ensuremath{\mathsf{SA}}}
\newcommand{\AO}{\ensuremath{\mathsf{AO}}}

\newcommand{\tup}[1]{\ensuremath{\langle #1 \rangle}\xspace}
\newcommand{\cond}[2]{\ensuremath{#1 \triangleright #2}\xspace}
\newcommand{\condz}[3]{\ensuremath{\forall \vec{#1}{:} #2\triangleright #3}\xspace}
\newcommand{\set}[1]{\ensuremath{\{#1\}}\xspace}

\newcommand{\constrThree}{\ensuremath{\mathcal{C}}\xspace}

\newtheoremstyle{mytheorem}
    {3pt}
    {3pt}
    {\normalfont}
    {0pt}
    {\bfseries}
    {.}
    { }
    {}
\theoremstyle{mytheorem}

\newtheorem{definition}{Definition}

\newtheorem{example}{Example}

\makeatletter
\let\OldStatex\Statex
\renewcommand{\Statex}[1][3]{%
  \setlength\@tempdima{\algorithmicindent}%
  \OldStatex\hskip\dimexpr#1\@tempdima\relax}
\makeatother
\algrenewcommand\algorithmicindent{1.0em}%

\newcommand{\compileC}{\textsc{CompileC}\xspace}

\def\tcore{TCORE\xspace}
\def\lccr{Lifted\tcore}
\def\lcc{LCC\xspace}
\def\endaction{\ensuremath{fin}\xspace}



\usepackage{tikz}
\usetikzlibrary{circuits.ee.IEC}
\newcommand\Ground{%
\mathbin{\text{\begin{tikzpicture}[circuit ee IEC,yscale=0.4,xscale=0.3]
\draw (0,2ex) to (0,0) node[ground,rotate=-90,xshift=.65ex] {};
\end{tikzpicture}}}%
}

\lstdefinestyle{customstyle}{
    backgroundcolor=\color{gray!10},
    basicstyle=\ttfamily\footnotesize,
    breaklines=true,
    frame=single,
    rulecolor=\color{black},
    keywordstyle=\color{blue},
    commentstyle=\color{green!50!black},
    stringstyle=\color{red}
}
\everymath{\it}\everydisplay{\it}

\def\ground{\Ground {}}

\def\groundcond{c^{\ground}}

\def\regression{R}
\def\Lregression{\regression^{L}}
\def\LGamma{\Gamma^L}

\newcommand{\compileCR}{\textsc{CompileCRegression}\xspace}

\def\weakestCcond{w^{c}}
\def\weakestC{w}

\DeclareMathOperator{\val}{=}  
\DeclareMathOperator{\nval}{{\neq}}  
\def\plus{{+}}
\def\minus{{-}}

\def\actioninst{a^{\ground}}

\def\subst{\theta}

\def\substainst{\subst_{\actioninst}}
\def\substz{\subst_{z}}
\def\substzfree{\subst_{\zfree}}
\newcommand{\substzi}[1]{\subst_{z_{#1}}}
\def\substf{\subst_{f}}
\def\substfa{\subst_{f\actioninst}}
\def\substpsi{\subst_{\psi}}
\def\substpsia{\subst_{\psi\actioninst}}
\newcommand{\substza}[1]{\subst_{z_{#1}\actioninst}}
\newcommand{\applys}[1]{|_{#1}}
\newcommand{\applyai}[1]{\applys{\subst_{#1}}}

\def\qeddef{\hfill $\square$}
\def\qedex{\hfill $\lozenge$}
\def\qedprop{\hfill $\lozenge$}

\def\pddlProblem{\Pi} 

\def\pddlAtoms{F}
\def\pddlActions{A}
\def\pddlInit{I}
\def\pddlGoal{G}
\def\pddlObjects{O}
\def\pddlConstraints{C}
\def\constr{q}

\def\pddlActionArgs{Arg}
\newcommand{\actionArgs}[1]{\pddlActionArgs(#1)}
\def\pddlPreconditions{Pre}
\newcommand{\actionPrec}[1]{\pddlPreconditions(#1)}
\def\pddlEffects{Eff}
\newcommand{\actionEff}[1]{\pddlEffects(#1)}
\newcommand{\actionEffPlus}[1]{\pddlEffects^{+}(#1)}
\newcommand{\actionEffMinus}[1]{\pddlEffects^{-}(#1)}

\def\Always{\mathsf{Always}}
\def\Sometime{\mathsf{Sometime}}
\def\AtMostOnce{\mathsf{AtMostOnce}}
\def\SometimeBefore{\mathsf{SometimeBefore}}
\def\SometimeAfter{\mathsf{SometimeAfter}}

\def\stateseq{\sigma}
\def\plan{\pi}

\def\zfree{\vec{z}^f}

\def\symk{\textsc{SymK}}
\def\fd{\textsc{FD}\xspace}

\def\blocksworldtwo{\textsc{Blocksworld2}}

\def\clearF{\mathtt{clear}}
\def\onF{\mathtt{on}}
\def\atF{\mathtt{at}}
\def\ontableF{\mathtt{onTable}}
\def\handemptyF{\mathtt{handEmpty}}
\def\holdingF{\mathtt{holding}}
\def\towerbaseF{\mathtt{towerBase}}
\def\putdown{\mathtt{putdown2}}
\def\pickup{\mathtt{pickup2}}
\def\rthree{\mathtt{R3}}
\def\ltwothree{\mathtt{L23}}
\def\lthreetwo{\mathtt{L32}}

\def\addlegendimage{\csname pgfplots@addlegendimage\endcsname}

\pgfplotsset{compat=newest}
\pgfplotsset{tick label style={/pgf/number format/fixed}}

\definecolor{bblue}{rgb}{0.2,0.4,0.8}
\definecolor{rred}{rgb}{0.8,0.2,0.2}
\definecolor{ggreen}{rgb}{0.2,0.7,0.3}

\pgfplotsset{
    myKTickLabel/.style={
        xticklabel={
            \pgfmathparse{
                \tick >= 1e6 ? int(\tick/1e6) :
                (\tick >= 1e3 ? int(\tick/1e3) : int(\tick))
            }%
            \pgfmathresult%
            \pgfmathparse{
                \tick >= 1e6 ? "M" :
                (\tick >= 1e3 ? "K" : "")
            }%
            \pgfmathresult
        }
    }
}
\definecolor{lightblue}{HTML}{A4DBE8}
\title{Two Constraint Compilation Methods for Lifted Planning}
\author{
    Periklis Mantenoglou\textsuperscript{\rm 1}, 
    Luigi Bonassi\textsuperscript{\rm 2},
    Enrico Scala\textsuperscript{\rm 3},
    Pedro Zuidberg Dos Martires\textsuperscript{\rm 1}
}
\affiliations{
    \textsuperscript{\rm 1}Örebro University, Sweden\\
    \textsuperscript{\rm 2}Oxford Robotics Institute, University of Oxford, UK\\
    \textsuperscript{\rm 3} University of Brescia, Italy\\

    periklis.mantenoglou@oru.se, luigibonassi@robots.ox.ac.uk, enrico.scala@unibs.it, pedro.zuidberg-dos-martires@oru.se
}

\usepackage{xcolor}
\usepackage{thm-restate}

\begin{document}

\maketitle

\begin{abstract}
    We study planning in a fragment of PDDL with qualitative state-trajectory constraints, capturing safety requirements, task ordering conditions, and intermediate sub-goals commonly found in real-world problems.
    A prominent approach to tackle such problems is to compile their constraints away, leading to a problem that is supported by state-of-the-art planners. 
    Unfortunately, existing compilers do not scale on problems with a large number of objects and high-arity actions, as they necessitate grounding the problem before compilation.
    To address this issue, we propose two methods for compiling away constraints without grounding, making them suitable for large-scale planning problems.
    We prove the correctness of our compilers and outline their worst-case time complexity.
    Moreover, we present a reproducible empirical evaluation on the domains used in the latest International Planning Competition.
    %
    Our results demonstrate that our methods are efficient and produce planning specifications that are orders of magnitude more succinct than the ones produced by compilers that ground the domain, while remaining competitive when used for planning with a state-of-the-art planner.
    %
    %
    %
\end{abstract}


\section{Introduction}\label{sec:intro}

Planning involves identifying a sequence of actions that brings about a desired goal state, and is vital in several real-world tasks~\cite{Ropero2017AVR,DBLP:conf/iccad/ShaikP23}. 
%
In addition to final goals, a planning task may require the satisfaction of certain temporal constraints throughout the execution of the plan~\cite{DBLP:journals/amai/BacchusK98}. These constraints may enforce, e.g., a set of conditions that must be avoided~\cite{DBLP:conf/aaai/Steinmetz0KB22}, or the ordered resolution of certain tasks~\cite{DBLP:journals/jair/HoffmannPS04}, both of which are common in planning problems~\cite{DBLP:conf/iclr/JackermeierA25}. Many planning approaches supporting these types of specifications have been considered in the literature \cite{DBLP:journals/amai/BacchusK98,DBLP:conf/ijcai/PatriziLGG11,micheli:19:tpack,DBLP:conf/aaai/MallettTT21,bonassi:22:pac}. 

In this paper, we focus on the fragment of PDDL that introduced and standardised constraints over trajectories of states~\cite{DBLP:journals/ai/GereviniHLSD09}. Some approaches handle these types of constraints directly in the search engine~\cite{DBLP:conf/aips/BentonCC12,DBLP:conf/ijcai/HsuWHC07,DBLP:conf/aips/Edelkamp06}, while others remove the constraints using compilation approaches~\cite{DBLP:conf/aips/PercassiG19, DBLP:conf/aaai/BonassiGS24}. Alternatively, these constraints can be translated into Linear Temporal Logic (LTL) \cite{pnuelli:77:ltl}, or one of its variants \cite{DBLP:conf/ijcai/GiacomoSFR20}, and handled by existing compilers \cite{DBLP:conf/aaai/BaierM06,DBLP:conf/ijcai/TorresB15,DBLP:conf/aips/BonassiGFFGS23}. The compilation step yields an equivalent constraint-free problem that can be solved by off-the-shelf planners~\cite{DBLP:journals/jair/Helmert06,DBLP:conf/aaai/SpeckMN20}.

For PDDL qualitative state-trajectory constraints, TCORE~\cite{DBLP:conf/aips/BonassiGPS21} constitutes the state-of-the-art compilation approach.
Specifically, TCORE employs a so-called regression operator~\cite{DBLP:conf/ecai/Rintanen08} that expresses the circumstances under which an action may affect constraint satisfaction.
%
%
%
However, TCORE's regression operator only works on ground actions, requiring compilation over the fully grounded problem, which can be orders of magnitude larger than the original (non-ground) problem.

Typically, the cost of grounding increases with the number of objects.
This becomes a critical issue in domains with high-arity actions, which have to be grounded once for each valid combination of objects in their arguments. Consider, e.g., the following variant of the \textsc{Blocksworld} domain.

%

\begin{example}[\blocksworldtwo]\label{ex:running}
%
\textsc{Blocksworld} involves rearranging stacks of blocks on a table to reach a final configuration by moving one block at a time.
States are described by atoms such as $\clearF(b)$, expressing that there is no block on top of block $b$, and $\onF(b_1,b_2)$, expressing that block $b_1$ is on top of block $b_2$. 
\blocksworldtwo\ extends \textsc{Blocksworld} by allowing operations on towers consisting of two blocks. 
Action $\pickup$, e.g., allows the agent to pick up a two-block tower.
Grounding a \blocksworldtwo\ problem introduces one ground instance of the $\pickup$ action for each (ordered) pair of blocks in the problem. 
\qedex
\end{example}

Although large instances of such domains can be handled by lifted planners~\cite{DBLP:conf/ijcai/CorreaG24, DBLP:conf/aaai/HorcikSSP25,DBLP:conf/aips/HorcikF21, DBLP:conf/ecai/HorcikF23, DBLP:conf/aips/Fiser23,  DBLP:conf/aaai/ChenTT24}, these approaches do not consider constraints. In \blocksworldtwo, these constraints might be, for example, that block $b_1$ must never be on the table, or that placing $b_1$ on top of $b_2$ is only allowed after placing $b_3$ on the~table.

To address this issue, we propose two constraint compilers that are \textit{lifted}, in the sense that they avoid unnecessary grounding.
We summarise our \textbf{contributions} as follows.
\textbf{First}, we propose \lccr, a constraint compiler that employs a lifted variant of the regression operator used by TCORE
in order to bypass grounding.
\textbf{Second}, we propose the Lifted Constraint Compiler (\lcc) that compiles away constraints without grounding and without the use of a regression operator.
\textbf{Third}, we demonstrate the correctness of \lccr and \lcc, and outline their worst-case time complexity.
\textbf{Fourth}, we present an empirical evaluation on challenging planning domains from the latest International Planning Competition~\cite{DBLP:journals/aim/TaitlerAEBFGPSSSSS24}.
Our results demonstrate that \lccr\ and \lcc\ are efficient and lead to compiled specifications that are significantly more succinct than the ones produced by compilers that ground the domain, while remaining competitive when used for planning with a state-of-the-art planner.

The proofs of all propositions are provided in Appendices \ref{proofs_lccr} and \ref{proofs_lcc}.
Appendices \ref{sec:appendix:experiments:constraints} and \ref{sec:appendix:experiments:icaps21} provide details on the benchmark we employed and additional experimental results.

\section{Background}


\subsection{Planning with Constraints in PDDL}
\label{sef:constraintspddl}

A planning problem is a tuple $\pddlProblem\val \tup{\pddlAtoms, \pddlObjects, \pddlActions, \pddlInit, \pddlGoal, \pddlConstraints}$, where $\pddlAtoms$ is a set of atoms, $\pddlObjects$ is a set of objects, $\pddlActions$ is a set of actions, $\pddlInit\subseteq \pddlAtoms$ is an initial state, $\pddlGoal$ is a formula over $\pddlAtoms$ denoting the goal of the problem, and $\pddlConstraints$ is a set of constraints. 
Each action $a\in\pddlActions$ comprises a list of arguments $\actionArgs{a}$, a precondition $\actionPrec{a}$, which is a formula over $\pddlAtoms$, and a set of conditional effects $\actionEff{a}$.
Each conditional effect in $\actionEff{a}$ is an expression $\forall \vec{z}{:}\ c\rhd e$, where $c$ is a formula, $e$ is a literal---both constructed based on the atoms in $\pddlAtoms$---and $\vec{z}$ is a set of variables.
$\condz{z}{c}{e}$ expresses that, for each possible substitution $\substz$ of the variables in $\vec{z}$ with objects from $\pddlObjects$, if the condition $c\applys{\substz}$ is true, then effect $e\applys{\substz}$ is brought about~\cite{DBLP:journals/aim/McDermott00}, where $x\applys{\theta}$ denotes the application of substitution $\theta$ in formula $x$.
We use $\actionEffPlus{a}$ (resp.~$\actionEffMinus{a}$) to denote the set of effects of an action $a$ where literal $e$ is positive (negative), and $\actioninst$ to denote an action whose arguments are grounded to objects in $\pddlObjects$, i.e., $\actionArgs{\actioninst}$ does not contain variables.
A state $s\subseteq\pddlAtoms$ contains the atoms that are true in $s$.
A ground action $\actioninst$ is applicable in state $s$ if $s\models \actionPrec{\actioninst}$, and its application yields state ${s[\actioninst]\val (s\setminus \bigcup_{c\rhd e\in \actionEffMinus{\actioninst}: s\models c} e)\cup \bigcup_{c\rhd e\in \actionEffPlus{\actioninst}: s\models c} e}$. 

We focus on a fragment of PDDL that includes $\Always$, $\Sometime$, $\AtMostOnce$, $\SometimeBefore$ and $\SometimeAfter$ constraints~\cite{DBLP:journals/ai/GereviniHLSD09}.
%
Given a sequence of states $\stateseq\val \tup{s_0, \dots, s_n}$ and first-order formulae $\phi$ and $\psi$, these constraint types are defined as follows:
\begin{compactitem}
    \item $\stateseq\models\Always(\phi)$ (or $\A(\phi)$) iff $\forall i{:}\ 0\leq i \leq n, s_i\models\phi$, i.e., $\phi$ holds in every state of $\stateseq$.
    \item $\stateseq\models\Sometime(\phi)$ ($\ST(\phi)$) iff $\exists i{:}\ 0\leq i \leq n, s_i\models\phi$, i.e., $\phi$ holds in at least one state of $\stateseq$.
    \item $\stateseq\models\AtMostOnce(\phi)$ ($\AO(\phi)$) iff $\forall i{:}\ 0\leq i\leq n$, if $s_i\models\phi$, then $\exists j{:}\ j\geq i$ such that $\forall k{:}\ i\leq k\leq j, s_k\models \phi$ and $\forall k{:}\ j<k\leq n, s_k\models \neg \phi$, i.e., $\phi$ is true in at most one continuous subsequence of $\stateseq$.
    \item $\stateseq\models\SometimeBefore(\phi, \psi)$ ($\SB(\phi, \psi)$) iff $\forall i{:}\ 0\leq i\leq n$, if $s_i\models \phi$ then $\exists j{:}\ 0\leq j< i, s_j\models \psi$, i.e., if there is a state $s_i$ where $\phi$ holds, then there is a state that is before $s_i$ in $\stateseq$ where $\psi$ holds.
    \item $\stateseq\models\SometimeAfter(\phi, \psi)$ ($\SA(\phi, \psi)$) iff $\forall i{:}\ 0\leq i\leq n$, if $s_i\models \phi$ then $\exists j{:}\ i\leq j\leq n, s_j\models \psi$, i.e., if there is a state $s_i$ where $\phi$ holds, then $\psi$ also holds at state $s_i$ or there is a state that is after $s_i$ in $\stateseq$ where $\psi$ holds.
\end{compactitem}

A plan $\plan$ for problem $\pddlProblem\val \tup{\pddlAtoms, \pddlObjects, \pddlActions, \pddlInit, \pddlGoal, \pddlConstraints}$ is a sequence of ground actions $\tup{\actioninst_0, \dots, \actioninst_{n\minus 1}}$ from $\pddlActions$ and $\pddlObjects$.
Plan $\plan$ is valid iff, for the sequence $\sigma\val \tup{s_0, \dots, s_n}$ such that $s_0\val I$ and  $s_{i\plus 1}\val s_i[\actioninst_i]$, $\forall i{:}\ 0\leq i< n$, we have $s_i\models \actionPrec{\actioninst_i}$, $\forall i{:}\ 0\leq i< n$, $s_n\models \pddlGoal$, and $\forall \constr\in\pddlConstraints{:}\ \sigma\models\constr$.
%
%

\subsection{Regression}

The regression of a formula $\phi$ through an action $a$ is the weakest condition that must hold to guarantee the satisfaction of $\phi$ after the application of $a$. Regression has been used in different contexts, e.g., situation calculus \cite{situation-regression}, numeric planning \cite{numeric-regression}, non deterministic planning \cite{DBLP:conf/ecai/Rintanen08}, and SAS$^+$ planning \cite{sas-plus-regression}. We focus on regression for classical planning as defined by \citet{DBLP:conf/ecai/Rintanen08} for ground $\phi$ and $\actioninst$.
\begin{definition}[Regression Operator]\label{def:regression}
    Consider a ground formula $\phi$ and a ground action $\actioninst$.
    Regression $R(\phi, \actioninst)$ is the formula obtained from $\phi$ by replacing every atom $f$ in $\phi$ with $\Gamma_f(\actioninst)\vee (f\wedge \neg \Gamma_{\neg f}(\actioninst))$, where the gamma operator $\Gamma_l(\actioninst)$ for a literal $l$ is defined as: $\Gamma_l(\actioninst)\val\bigvee_{\cond{c}{l}\in \actionEff{\actioninst}} c$. 
    \qeddef
\end{definition}

TCORE employs regression to compile away constraints~\cite{DBLP:conf/aips/BonassiGPS21}.
To do this, TCORE first grounds the problem and then calculates $R(\phi, \actioninst)$ for each formula $\phi$ appearing in a constraint and each action $\actioninst$.
Subsequently, TCORE introduces $R(\phi, \actioninst)$ formulae and `monitoring atoms'---whose purpose is to track the status of constraints---in action preconditions and effects, capturing the semantics of the constraints in the compiled problem.
%
%
%
%

\section{\lccr}\label{sec:lccr}


\subsection{Lifted Regression}
The idea behind constructing \lccr is to generalize the regression operator from the ground domain to the lifted domain.
Consequently, \lccr will need to compute the lifted regression operator $\Lregression(\phi,a)$ for each action $a$ and each formula $\phi$ appearing in an argument of a constraint.
We achieve this by constructing a lifted gamma operator $\LGamma_l(a)$ and adapting Definition~\ref{def:regression} accordingly.
%

Similar to $\Gamma_l(a)$, $\LGamma_l(a)$ expresses the weakest condition that leads to the satisfaction of $l$ after applying action $a$, with the difference of $\LGamma_l(a)$ operating on non-ground $l$ and $a$.
%
%

We start with the case of a formula $\phi$ whose truth value may be affected only by $\cond{c}{e}$ effects of action $a$, i.e., the only set to consider in $\condz{z}{c}{e}$ is the empty set ($\vec{z}=\emptyset$).
%
For each literal $l$ in $\phi$ and each effect $\cond{c}{e}$ of $a$, we compute the most general unifier $\xi(l,e)$ between $l$ and $e$ (if any) and derive the weakest condition $\weakestCcond_l(\cond{c}{e})$ under which an action with effect $\cond{c}{e}$ may bring about $l$.
We derive $\xi(l, e)$ via \citeauthor{DBLP:journals/jacm/Robinson65}'s resolution algorithm~\citep{DBLP:journals/jacm/Robinson65} by using the most general resolution, in the sense that variables are grounded to constants only when required.
%


\begin{definition}[Weakest Condition $\weakestCcond_l(\cond{c}{e})$] \label{def:wce}
The weakest condition $\weakestCcond_l(\cond{c}{e})$ under which effect $\cond{c}{e}$ brings about literal $l$ is:
\begin{align*}
\weakestCcond_l(\cond{c}{e})\val
\begin{cases}
c \wedge  \bigwedge_{\substack{(t_i\doteq u_i)\in \xi(l,e)}} (t_i=u_i) & \text{if } \exists \xi(l, e)\\
\bot & \text{if } \nexists \xi(l, e)
\end{cases}
\end{align*}
where $t_i$ and $u_i$ denote the arguments of the literals $l$ and $e$ respectively, and $t_i\doteq u_i$ denotes their unification.
\qeddef
\end{definition}
\begin{example}[Weakest Condition $\weakestCcond_l(\cond{c}{e})$]\label{ex:wce}
Consider the \blocksworldtwo\ domain, and action $\putdown$, which allows the agent to place a block or a two-block tower on the table.
$\putdown$ has one argument $b$ and its effects are: 
\begin{compactitem}
\item $\handemptyF$, i.e., the agent's hand is empty,
\item $\ontableF(b)$, i.e., block $b$ is on the table,
\item $\neg \holdingF(b)$, i.e., the agent is not holding $b$,
\item $\cond{\neg \towerbaseF(b)}{\clearF(b)}$, i.e., if $b$ is not the base of a block tower, then $b$ is clear. 
\end{compactitem}
Consider literals $\ontableF(b_1)$ and $\clearF(b_5)$, where $b_1$ and $b_5$ are ground.
$\ontableF(b_1)$ can only be unified with effect $\ontableF(b)$ of $\putdown$ via  $b\doteq b_1$, while $\clearF(b_5)$ can only be unified with literal $\clearF(b)$ of effect $\cond{\neg \towerbaseF(b)}{\clearF(b)}$ via $b\doteq b_5$.
Thus, we have:
\begin{align*}
&\weakestCcond_{\ontableF(b_1)}(\ontableF(b))\val (b\val b_1)\\
&\weakestCcond_{\clearF(b_5)}(\cond{\neg \towerbaseF(b)}{\clearF(b)})\val \\
&\hspace{100pt}\neg \towerbaseF(b) \wedge (b\val b_5) \tag*{\qedex}
\end{align*}
%
%
%
%
\end{example}

%
Next, we extend Definition \ref{def:wce} to handle $\condz{z}{c}{e}$ effects for a non-empty set $\vec{z}$.
In the case of such an effect, a variable $u$ appearing in an argument of $e$ may be: (i) a parameter of $a$, or (ii) a variable in $\vec{z}$.
Definition \ref{def:wce} handles case (i) by introducing equalities between the action parameters in $e$ and the corresponding arguments of $l$.
Case (ii), however, needs a different type of treatment as the scope of a variable $u\in \vec{z}$ is only within $\condz{z}{c}{e}$, and thus, if we were to introduce an equality between $u$ and an argument of $l$ in $\Lregression_{l}(a)$, then $u$ would be a free variable in $\Lregression_{l}(a)$.

To tackle this issue, we leverage the semantics of $\condz{z}{c}{e}$, according to which $\condz{z}{c}{e}$ is equivalent to introducing one effect $\cond{c\applys{\substz}}{e\applys{\substz}}$ for each possible substitution $\substz$ of the variables in $\vec{z}$ with domain objects~\cite{DBLP:journals/aim/McDermott00}. 
As a result, $e$ can be unified with $l$ iff there is a substitution $\substz(l, e)$ such that $e\applys{\substz(l, e)}$ can be unified with $l$.
If there is such a substitution, then, in order for effect $\condz{z}{c}{e}$ to bring about $l$, there needs to be an assignment to the variables of $\vec{z}$ that appear in $c$ but not in $e$, i.e., the variables in set $\zfree$, such that $c\applys{\substz(l,e)}$ becomes true.
In other words, formula $\exists \zfree\ c\applys{\theta_z(l,e)}$ needs to hold.
Then, the weakest condition $\weakestC_l(\condz{z}{c}{e})$ under which $\condz{z}{c}{e}$ brings about $l$ requires for $\exists \zfree\ c\applys{\theta_z(l,e)}$ to hold under a variable assignment that unifies $l$ and $e\applys{\substz(l, e)}$.

\begin{definition}[$z$-substitution]
Consider a literal $l$ and a conditional effect $\forall \vec{z}{:}\ c\rhd e$, such that $l$ and $e$ can be unified and $\xi(l,e)$ is their most general unifier.
We define the $z$-substitution $\theta_z(l,e)$ of $l$ and $e$ as:
\begin{align*}
\theta_z(l,e)\val \{u_i\mapsto t_i\ |\ (t_i\doteq u_i)\in\xi(l,e) \wedge u_i\in \vec{z}\}
\end{align*}
where $x\mapsto y$ denotes that term $y$ substitutes term $x$. 
\qeddef
\end{definition}
\begin{definition}[Weakest Condition $\weakestC_l(\condz{z}{c}{e})$]\label{def:wzce}
The weakest condition under which $\condz{z}{c}{e}$ brings about $l$ is:
\begin{align*}
\weakestC_l(\condz{z}{c}{e})\val
\begin{cases}
\exists \zfree\ c\applys{\theta_z(l,e)} \wedge \bigwedge\limits_{\mathclap{(t_i\doteq u_i)\in \xi(l,e)\wedge u_i\notin \vec{z}}} (t_i=u_i) & \text{if } \exists \xi(l, e)\\
\bot & \text{if } \nexists \xi(l, e)
\end{cases}
\end{align*}
where $\vec{z}^{f}\val \{z\in\vec{z}\ |\ \nexists t_i: (z\mapsto t_i)\in\theta_z(l,e)\}$. 
\qeddef
\end{definition}

According to Definition \ref{def:wzce}, when $\vec{z}$ is empty, $\weakestC_l(\condz{z}{c}{e})$ coincides with $\weakestCcond_l(\cond{c}{e})$.

\begin{example}[Weakest Condition $\weakestC_l(\condz{z}{c}{e})$]\label{ex:wzce}
Consider that literal $l$ is $\clearF(b_5)$ and suppose that action $\putdown$ additionally has the following effect: 
\begin{compactitem}
\item $\forall \{topb\}{:}\ \cond{\onF(topb,b)}{\clearF(topb)}$, i.e., any block $topb$ that is on top of block $b$ is clear.
\end{compactitem}
$\clearF(b_5)$ unifies with $\clearF(topb)$ via ${topb\doteq b_5}$, and, since $topb$ is a $\vec{z}$ variable of the effect, we have $\theta_z(\clearF(b_5), \clearF(topb))\val \{topb\mapsto b_5\}$, while $\zfree$ is empty. 
Therefore, we have:
\begin{align*}
\weakestC_l(\forall \{topb\}{:}\ \cond{\onF(topb,b)}{\clearF(topb)})\val \onF(b_5, b) \tag*{\qedex}
\end{align*}
\end{example}

\begin{definition}[Lifted Gamma Operator $\LGamma_l(a)$]\label{def:lifted_gamma}
Given an action $a$ and a literal $l$, the lifted gamma operator is defined as: $\LGamma_{l}(a)\val \bigvee_{\forall \vec{z}{:}\cond{c}{e}\in \actionEff{a}} \weakestC_l(\condz{z}{c}{e})$. 
\qeddef
\end{definition}

According to Definition \ref{def:lifted_gamma}, $\LGamma_l(a)$ is the weakest condition under which action $a$ brings about $l$ via one of its effects.

\begin{example}[Lifted Gamma Operator $\LGamma_l(a)$]\label{ex:lifted_gamma}
Based on the results in Examples \ref{ex:wce} and \ref{ex:wzce}, We have:
\begin{align*}
\LGamma_{\ontableF(b_1)}(\putdown)\val &(b\val b_1)\\
\LGamma_{\clearF(b_5)}(\putdown)\val &(\neg \towerbaseF(b) \wedge (b\val b_5)) \\
&\vee \onF(b_5, b) \tag*{\qedex}
\end{align*}
\end{example}

\begin{definition}[Lifted Regression Operator $\Lregression(\phi, a)$]\label{def:lifted_regression}
    Consider a first-order formula $\phi$ and an action $a$.
    The lifted regression $\Lregression(\phi, a)$ is the formula obtained from $\phi$ by replacing every atom $f$ in $\phi$ with $\LGamma_f(a)\vee (f\wedge \neg \LGamma_{\neg f}(a))$.
    \qeddef
\end{definition}

According to Definition \ref{def:lifted_regression}, $\Lregression(\phi, a)$ is the weakest condition under which $\phi$ is true after the execution of action $a$.

\begin{example}[Lifted Regression $\Lregression(\phi, a)$]\label{ex:lifted_regression}
Based on the results in Example \ref{ex:lifted_gamma}, and since $\LGamma_{\neg \ontableF(b_1)}(\putdown)\val \bot$ and $\LGamma_{\neg \clearF(b_5)}(\putdown)\val \bot$, we have:
\begin{align*}
&\Lregression(\ontableF(b_1), \putdown)\val (b\val b_1) \vee \ontableF(b_1)\\
&\Lregression(\clearF(b_5), \putdown)\val \onF(b_5, b) \vee \clearF(b_5) \vee \\
&\hspace{100pt}(\neg \towerbaseF(b) \wedge (b\val b_5)) \tag*{\qedex}
\end{align*}
\end{example}

\subsection{The \lccr Compiler}
\lccr follows the same steps as the TCORE compiler~\cite{DBLP:conf/aips/BonassiGPS21}, with the exception of using lifted regression.
As in TCORE, the main intuition behind \lccr is to foresee that an action $a$ will affect the truth value of a formula $\phi$ appearing in a constraint by checking whether $\Lregression(\phi, a)$ holds.
If it does hold, then \lccr sets a so-called `monitoring atom', in order to express the status of $\phi$ after an execution of $a$.
Contrary to TCORE, \lccr introduces the necessary regression formulae and monitoring atoms without grounding the problem, thus avoiding the combinatorial explosion induced by grounding. 

Algorithm \ref{alg:lifted_tcore} outlines the steps of \lccr. 
First, we check if an $\A$ or $\SB$ constraint (c.f. Section~\ref{sef:constraintspddl}) is violated in the initial state (line \ref{line:lccr_check_unsolvable}).
Since their violation is irrevocable, we deem the problem unsolvable (line \ref{line:lccr_unsolvable}).
Next, we introduce the monitoring atoms required to track the constraints of the problem (line \ref{line:lccr-F'}). 
$hold_c$ expresses that constraint $c$ is satisfied based on the state trajectory induced so far.
$seen_\phi$ expresses that there is a state in the induced trajectory where $\phi$ holds.
$hold_c$ atoms are used to monitor $\ST$ and $\SA$ constraints, while $seen_\phi$ atoms are required for $\SB$ and $\AO$ constraints.
Constraints of type $\A$ do not demand monitoring atoms.
We augment the initial state $I$ with the monitoring atoms that are satisfied in $I$ (line \ref{line:lccr-I'}).

Afterwards, for each action $a$, we compute the preconditions and the effects that need to be added in $a$ in order to capture the semantics of the constraints within the compiled problem (lines \ref{line:lccr_foraction}--\ref{line:lccr_addeff}).
For a $\A(\phi)$ constraint, we need to add precondition $\Lregression(\phi, a)$ in $a$, stating that $\phi$ needs to be true after the execution of $a$ (line \ref{line:lccr_case_a}).
In the case of $\ST(\phi)$, we activate monitoring atom $hold_{\hspace{1pt}\ST(\phi)}$ if $\phi$ is true after executing $a$ (line \ref{line:lccr_case_st}).
For $\AO(\phi)$, we set $seen_\phi$ when $a$ brings about $\phi$ (line \ref{line:lccr_case_ao_e}) and forbid the execution of $a$ when $seen_\phi$ holds, $\phi$ is false and $a$ would bring about $\phi$ (line \ref{line:lccr_case_ao_p}), thus prohibiting $\phi$ from occurring more than once.
For $\SB(\phi, \psi)$, we set $seen_\psi$ when $a$ brings about $\psi$, and forbid the execution of $a$ when it would bring about $\phi$ while $seen_\phi$ does not hold.
For $\SA(\phi, \psi)$, we set $hold_{\hspace{1pt}\SA(\phi, \psi)}$ if $a$ brings about $\psi$, and $\neg hold_{\hspace{1pt}\SA(\phi, \psi)}$ if $a$ brings about a state where $\phi\wedge \neg \psi$ holds.
Lastly, we require that all $hold_c$ monitoring atoms need to be true in a goal state, reflecting that all $\ST$ and $\SA$ must have been satisfied by the end of the plan (line \ref{line:lccr_return}).

\begin{algorithm}[t]
\caption{\lccr}\label{alg:lifted_tcore}
\small

\begin{algorithmic}[1]
    \Require Planning problem $\Pi = \tup{F, \pddlObjects, A, I, G, \constrThree}$.
    \Ensure Planning problem $\Pi' = \tup{F', \pddlObjects, A', I', G', \emptyset}$

    \If{$\exists \A(\phi)\in\constrThree: I\models\neg \phi \vee \exists \SB(\phi, \psi)\in\constrThree: I\models \phi$} \label{line:lccr_check_unsolvable}
    \State \textbf{return} Unsolvable Problem \label{line:lccr_unsolvable}
    \EndIf
    
    \State $F' \gets F \cup\ \bigcup\limits_{\mathclap{\substack{c : \ST(\phi) \in \constrThree\ \vee \\ c : \SA(\phi, \psi) \in \constrThree}}}\set{hold_c}  \; \cup\ \bigcup\limits_{\mathclap{\SB(\phi, \psi) \in \constrThree}} \set{seen_\psi} \cup\ \bigcup\limits_{\mathclap{\AO(\phi) \in \constrThree}}\set{seen_\phi}$ \label{line:lccr-F'}

    \State $I' {\gets} I \cup\ \bigcup\limits_{\mathclap{c : \ST(\phi) \in \constrThree}}\set{hold_c\ |\ I\models \phi} \cup\ \bigcup\limits_{\mathclap{c : \SA(\phi, \psi) \in \constrThree}}\set{hold_c\ |\ I\models \psi{\vee} \neg \phi}$ \label{line:lccr-I'}

    \Statex \hspace{-4pt}$\cup\ \bigcup\limits_{\mathclap{\SB(\phi, \psi) \in \constrThree}} \set{seen_\psi\ |\ I\models \psi} \cup\ \bigcup\limits_{\mathclap{\AO(\phi) \in \constrThree}}\set{seen_\phi\ |\ I\models \phi}$

    \ForAll{$ a \in A $}\label{line:lccr_foraction}
        \State $P, E \gets \compileCR(a, \constrThree)$ \label{line:lccr_compile}
        \State $\actionPrec{a}\gets \actionPrec{a} \wedge \bigwedge_{p \in P}p$ \label{line:lccr_addprec}
        \State $\actionEff{a}\gets \actionEff{a} \cup E$ \label{line:lccr_addeff}
    \EndFor
    
    \State\textbf{return} $\tup{F', \pddlObjects, A, I',  G \wedge \bigwedge_{hold_c \in F'} hold_{c}, \emptyset}$ \label{line:lccr_return}
    \vspace{5pt}
    \Function{\compileCR}{$a, \constrThree$}
    \State $P, E = \{\}, \{\}$ \label{line:lccr_initpe}
    \For{$c \in \constrThree $} \label{line:lccr_forc}
    
        \If{$c$ is $\A(\phi)$}~$P \gets P \cup \{\Lregression(\phi, a)\}$ \label{line:lccr_case_a}

        \ElsIf{$c$ is $\ST(\phi)$} \State $E \gets E \cup \{\cond{\Lregression(\phi, a)}{\set{hold_c}}\}$ \label{line:lccr_case_st}
            
        \ElsIf{$c$ is $\AO(\phi)$} \label{line:lccr_case_ao}
            \State $E \gets E \cup \{\cond{\Lregression(\phi, a)}{\set{seen_\phi}}\}$  \label{line:lccr_case_ao_e}
            \State $P \gets P \cup \{\neg(seen_\phi \wedge \neg\phi \wedge \Lregression(\phi, a))\}$ \label{line:lccr_case_ao_p}
             
        \ElsIf{$c$ is $\SB(\phi, \psi)$} \label{line:lccr_case_sb}
        \State $E \gets E \cup \{\cond{\Lregression(\psi, a)}{\set{seen_\psi }}\}$ \label{line:lccr_case_sb_e}
        \State $P \gets P \cup \{\Lregression(\phi, a) \rightarrow seen_{\psi}\}$ \label{line:lccr_case_sb_p}

        \ElsIf{$c$ is $\SA(\phi, \psi)$} \label{line:lccr_case_sa}
        \State $E \gets E \cup \set{\cond{\Lregression(\psi, a)}{\set{hold_c}}} $ \label{line:lccr_case_sa_e}

        \Statex \quad\quad\quad\hspace{3.7pt}$\cup \set{{\Lregression(\phi, a) {\wedge} \neg \Lregression(\psi, a){\rhd} \set{\neg hold_c}}}$
        \EndIf        
    \EndFor
    \State \textbf{return} $P$, $E$ \label{line:lccr_returnpe}
   \EndFunction 
\end{algorithmic}
\end{algorithm}


\begin{example}[\lccr]\label{ex:lccr}
Consider a \blocksworldtwo\ problem with the following constraints: $\ST(\clearF(b_5))$, $\AO(\ontableF(b_1))$, $\SB(\clearF(b_5), \exists topb: \onF(topb, b_3))$.
To compile away these constraints, \lccr\ introduces monitoring atoms $hold_{\hspace{1pt}\ST(\clearF(b_5))}$, $seen_\psi$ and $seen_\phi$, where $\psi$ is $\exists topb: \onF(topb, b_3)$ and $\phi$ is $\ontableF(b_1)$, and updates the initial state with these atoms according to line \ref{line:lccr-I'}.
Subsequently, \lccr introduces a set of action-specific preconditions and effects.
For action $\putdown$, e.g., \lccr adds preconditions $P$ and effects $E$, i.e.:
\begin{align*}
P\val \{ &\Lregression(\clearF(b_5),\putdown)\rightarrow seen_\psi, \\
         &\neg (seen_\phi\wedge \neg \phi \wedge \Lregression(\ontableF(b_1), \putdown)\}\\
E\val \{ &\cond{\Lregression(\clearF(b_5),\putdown)}{hold_{\hspace{1pt}\ST(\clearF(b_5))}}, \\
         &\cond{\Lregression(\ontableF(b_1), \putdown)}{seen_\phi}\}
\end{align*}
where expressions $\Lregression(\ontableF(b_1), \putdown)$ and $\Lregression(\clearF(b_5),\putdown)$ were derived in Example \ref{ex:lifted_regression}.
\qedex
\end{example}

According to line \ref{line:lccr_case_sb_e} of Algorithm \ref{alg:lifted_tcore}, set $E$ of Example~\ref{ex:lccr} should have included effect $\cond{\Lregression(\psi, \putdown)}{seen_\psi}$.
The reason for its omission was based on the fact that $\Lregression(\psi, \putdown)\val\psi$, which implies that $\psi$ holds after the execution of action $\putdown$ in a state $s$ iff $\psi$ holds at $s$.
Therefore, if $\Lregression(\psi, \putdown)\psi$ holds, then there is an earlier action $a$, where $\Lregression(\psi, a)\nval \psi$, that made $\psi$ true, and thus set atom $seen_\psi$.
Thus, it would be redundant to reinstate $seen_\psi$ when $\putdown$ is executed.
By omitting regression formulae such that $\Lregression(\phi, a)\val \phi$, \lccr compresses the compiled specifications without sacrificing correctness.

\subsection{Theoretical Properties}
We prove that \lccr\ is correct, i.e., a problem has the same solutions as its compiled version produced by \lccr, and outline its complexity.
Towards correctness, we first demonstrate a property of lifted regression.

\begin{restatable}[$\Lregression$ Correctness]{lemma}{propliftedregrcorrectness}
\label{prop:lifted_regr_correctness}
Consider a state $s$, an action $a$, a closed first-order formula $\phi$, and a ground incarnation $\actioninst$ of $a$ such that $\actioninst$ is applicable in $s$.
$\substainst$ denotes the substitution of the arguments of $a$ with those of $\actioninst$.
We have:
\begin{align*}
    s\models \Lregression(\phi,a)\applys{\substainst} \iff s[\actioninst]\models\phi \tag*{\qedprop}
\end{align*}
  \end{restatable}

\begin{restatable}[Correctness of \lccr]{proposition}{proplccrcorrectness}
\label{prop:lccr_correctness}
If \lccr compiles a problem $\pddlProblem$ into problem $\pddlProblem'$, then a plan $\plan$ is valid for $\pddlProblem$ iff $\plan$ is valid for $\pddlProblem'$. 
\qedprop
\end{restatable}

\begin{restatable}[Complexity of \lccr]{proposition}{proplccrcomplexity}
\label{prop:lccr_complexity}
Assuming that the nesting depth of quantifiers in constraint formulae is bounded by constant $b$, 
the worst-case time complexity of compiling a problem $\pddlProblem$ with \lccr is $\mathcal{O}(n_cn_f^b\plus n_a n_c n^2_f n_k)$, where $n_c$, $n_f$, $n_a$ and $n_k$ denote, respectively, the number of constraints, the number of atoms, the number of actions and the maximum atom arity in $\pddlProblem$.
\qedprop
\end{restatable}
%
Allowing an unbounded quantifier nesting in constraint formulae makes the step of checking whether the initial state models such a formula (line \ref{line:lccr_check_unsolvable}) PSPACE-complete~\cite{DBLP:conf/stoc/Vardi82}, dominating the complexity of \lccr.
In practice, however, the quantifier nesting depth is shallow.

\section{The Lifted Constraint Compiler}\label{sec:lcc}

%
Next, we discuss the `Lifted Constraint Compiler' (\lcc).
\lcc does not use lifted regression, and works by computing a set of preconditions $P$ and a set of effects $E$ that are \emph{independent} from actions.
Effects $E$ record the status of constraints: when an action $a$ is executed in a state $s$, $E$ will introduce in the next state $s'$ information regarding the status of the constraints in $s$.
The purpose of preconditions $P$ is to prevent the execution of further actions in a state $s$ when a constraint has been violated. 
These new preconditions and effects are shared among all actions.
In this way, \lcc monitors constraint violation without the need for regression, i.e., instead of foreseeing that an action $a$ will affect a formula $\phi$ of a constraint via the regression of $\phi$ through $a$, \lcc allows the application of any executable action $a$ but then blocks subsequent state expansion if $a$ led to constraint violation.
Following this schema, when a state $s$ where the goal is satisfied is reached, there has not been an earlier check on whether the constraints are satisfied in $s$.
To address this, \lcc introduces a new action $\endaction$ whose purpose is to verify that a state of the plan that satisfies the goal also satisfies the constraints.
To enforce that $\endaction$ is the last action executed, we use a new atom $end$ that marks the end of the plan.

%

\begin{algorithm}[t]
\caption{\lcc}
\label{alg:lifted_no_regr}
\begin{algorithmic}[1]
    \small
    \Require Planning Problem $\pddlProblem\val \tup{\pddlAtoms, \pddlObjects, \pddlActions, \pddlInit, \pddlGoal, \pddlConstraints}$.
    \Ensure Planning Problem $\pddlProblem\val \tup{\pddlAtoms', \pddlObjects, \pddlActions', \pddlInit', \pddlGoal', \emptyset}$.
    \State $F' \gets F \cup\ \bigcup\limits_{\mathclap{\substack{c : \ST(\phi) \in \constrThree\ \vee \\ c : \SA(\phi, \psi) \in \constrThree}}}\set{hold_c} \cup\bigcup\limits_{\mathclap{\AO(\phi) \in \constrThree}}\set{seen_\phi, prevent_\phi} \cup $ \label{line:lcc_fprime}

    \quad\quad\hspace{7pt}$\bigcup\limits_{\mathclap{\SB(\phi, \psi) \in \constrThree}} \set{seen_\psi} \cup \{end\}$    

    
    \State $P, E \gets \compileC(\constrThree)$ \label{line:lcc_compile}
    \ForAll{$ a \in A $}\label{line:lcc_foraction}
        \State $\actionPrec{a}\gets \actionPrec{a} \wedge \bigwedge_{p \in P}p \wedge \neg end$ \label{line:lcc_addprec}
        \State $\actionEff{a}\gets \actionEff{a} \cup E$ \label{line:lcc_addeff}
		
    \EndFor
    \State $\actionPrec{\endaction} \gets P \wedge \neg end$ \label{line:lcc_addendprec}
    \State $\actionEff{\endaction} \gets E \cup \{end\}$ \label{line:lcc_addendeff}
    \State $A' \gets A \cup \{\endaction\}$ \label{line:lcc_addendaction}
    \State \textbf{return} $\tup{F', \pddlObjects, A', I \cup \bigcup\limits_{\mathclap{\SA(\phi, \psi) \in \constrThree}}\set{hold_c},  G \wedge \bigwedge\limits_{\mathclap{\substack{c : \ST(\phi) \in \constrThree\ \vee \\ c : \SA(\phi, \psi) \in \constrThree}}} hold_{c} \wedge end, \emptyset}$ \label{line:lcc_return}
    \Function{\compileC}{$\constrThree$} \label{line:lcc_function}
    \State $P, E \gets \{\}, \{\}$ \label{line:lcc_initpe}
    \ForAll{$c \in \constrThree $} \label{line:lcc_forconstr}
    
        \If{$c$ is $\A(\phi)$}~$P \gets P \cup \{\phi\}$ \label{line:lcc_a_prec}

        \ElsIf{$c$ is $\ST(\phi)$}~$E \gets E \cup \{\cond{\phi}{hold_c}\}$  \label{line:lcc_st_eff}
            
        \ElsIf{$c$ is $\AO(\phi)$} \label{line:lcc_amo}
             \State $E \gets E \cup \{\cond{\phi}{seen_\phi}\}$  \label{line:lcc_amo_eff1}
             \State $E \gets E \cup \{\cond{(\neg \phi \wedge seen_\phi)}{prevent_\phi}\}$ \label{line:lcc_amo_eff2} 
             \State $P \gets P \cup \{\neg(\phi \wedge prevent_\phi)\}$ \label{line:lcc_amo_prec}
             
        \ElsIf{$c$ is $\SB(\phi, \psi)$} \label{line:lcc_sb}
             \State $E \gets E \cup \{\cond{\psi}{seen_\psi}\}$ \label{line:lcc_sb_eff}
              \State $P \gets P \cup \{\phi \rightarrow seen_{\psi}\}$ \label{line:lcc_sb_prec}

        \ElsIf{$c$ is $\SA(\phi, \psi)$} \label{line:lcc_sa}
             \State $E \gets E \cup \set{\cond{(\phi \wedge \neg \psi)}{\neg hold_{c}}} \cup \{\cond{\psi}{hold_c}\}$ \label{line:lcc_sa_eff}

        \EndIf        
    \EndFor
    \State \textbf{return} $P$, $E$ \label{line:lcc_returnpe}

\EndFunction

\end{algorithmic}

\end{algorithm}

Algorithm \ref{alg:lifted_no_regr} outlines the steps of \lcc. Initially, the compilation creates the necessary monitoring atoms to track the status of the constraints (line \ref{line:lcc_fprime}).
\lcc requires the same monitoring atoms used in \lccr, plus an additional atom $prevent_\phi$ for every $\AO(\phi)$.
%
This atom express that $\phi$ is false after having been true, and thus should be prevented from becoming true again in order to satisfy $\AO(\phi)$.
%

\par Next, \lcc determines the set of preconditions $P$ and the set of effects $E$ to be added to every action (line \ref{line:lcc_compile}) by iterating over each constraint $c$ (lines \ref{line:lcc_forconstr}-\ref{line:lcc_sa_eff}). 
If $c$ is $\A(\phi)$, \lcc adds precondition $\phi$ in all actions, ensuring that no progress can be made when we have $\neg \phi$ (line \ref{line:lcc_a_prec}).
If $c$ is $\ST(\phi)$, \lcc captures the satisfaction of $c$ by making $hold_c$ true when $\phi$ holds (line \ref{line:lcc_st_eff}).
If $c$ is $\AO(\phi)$, \lcc brings about $seen_\phi$ when $\phi$ is true (line \ref{line:lcc_amo_eff1}), and $prevent_\phi$ when $\phi$ is false after having been true in the past (line \ref{line:lcc_amo_eff2}).
Then, \lcc prevents the execution of any action when both $\phi$ and $prevent_\phi$ are true, a situation that violates $\AO(\phi)$ (line \ref{line:lcc_amo_prec}).
If $c$ is $\SB(\phi, \psi)$, \lcc sets $seen_\psi$ when $\psi$ becomes true, and prevents further actions when $\phi$ is true and $seen_\psi$ is false (lines \ref{line:lcc_sb_eff}-\ref{line:lcc_sb_prec}).
Lastly, if $c$ is $\SA(\phi, \psi)$, \lcc  activates $hold_c$ when $\psi$ is true, and deactivates it when $\phi$ is true and $\psi$ is false (line \ref{line:lcc_sa_eff}), thus expressing that $\SA(\phi, \psi)$ is violated only when $\phi$ held at some point in the past and $\psi$ has not become true since then.
%

Preconditions $P$ and effects $E$ are added to every action of the problem, including the newly introduced \endaction action (lines \ref{line:lcc_foraction}-\ref{line:lcc_addendeff}).
To ensure that no further action is executable after $\endaction$, we add precondition $\neg end$ to every action, and $end$ as an effect of \endaction. 
As a last step, \lcc determines the new goal, i.e., $G$ augmented with the conjunction of all $hold_c$ atoms and atom $end$, and the new initial state, where every $hold_c$ atom for a $\SA(\phi, \psi)$ constraint is set to true, maintaining correctness in the case where both $\phi$ and $\psi$ are always false.

\begin{example}[\lcc\ Compiler]\label{ex:lcc}
%
%
%
\lcc compiles the problem in Example \ref{ex:lccr} using the monitoring atoms $hold_{\ST(\clearF(b_5))}$, $seen_\psi$, $seen_\phi$ and $prevent_\phi$, where $\psi$ is $\exists topb: \onF(topb, b_3)$ and $\phi$ is $\ontableF(b_1)$, and then extends all actions with the following preconditions $P$ and effects $E$:
\begin{align*}
P\val \{ &\clearF(b_5)\rightarrow seen_\psi, \neg (\ontableF(b_1)\wedge prevent_\phi)\}\\
E\val \{ &\cond{\clearF(b_5)}{hold_{c_1}}, \cond{\exists topb: \onF(topb, b_3)}{seen_\psi}, \\
         &\cond{\ontableF(b_1)}{seen_\phi}, \\
         &\cond{(\neg \ontableF(b_1)\wedge seen_\phi)}{prevent_\phi}\} \tag*{\qeddef}
\end{align*}
\end{example}


\begin{restatable}[Correctness of \lcc]{proposition}{proplcccorrectness}
\label{prop:lcc_correctness}
If \lcc compiles a problem $\pddlProblem$ into problem $\pddlProblem'$, then plan $\tup{\actioninst_0, \dots, \actioninst_{n\minus 1}}$ is valid for $\pddlProblem$ iff plan $\tup{\actioninst_0, \dots, \actioninst_{n\minus 1}, \endaction}$ is valid for $\pddlProblem'$, and only plans ending with action $\endaction$ are valid for $\pddlProblem'$. 
\qedprop
\end{restatable}

\begin{restatable}[Complexity of \lcc]{proposition}{proplcccomplexity}\label{prop:lcc_complexity}
Suppose that $n_a$ and $n_c$ denote, respectively, the number of actions and the number of constraint in some problem $\pddlProblem$.
The worst-case time complexity of compiling $\pddlProblem$ with \lcc is $\mathcal{O}(n_a\plus n_c)$.
\qedprop
\end{restatable}






\section{Comparing \lccr and \lcc}
\label{sec:liftedtcore_vs_lcc}

We highlight the main differences between our compilers.

\textbf{Introduced Preconditions and Effects.}~
Through the use of regression, \lccr identifies action-specific preconditions and effects that capture the constraints, while \lcc adds the required preconditions and effects in all actions in a uniform way.
%
%
This action-specific compilation of \lccr allows us to avoid redundant updates to actions that are irrelevant to a constraint, leading to succinct compiled problems.
In contrast, the action-agnostic compilation of \lcc does not permit such optimisations.

\textbf{Constraint Evaluation Delay.}~\lcc monitors constraints with a 1-step delay, i.e., the status of constraint in a state $s$ can only be identified after an action $a$ is applied in $s$.
On the other hand, \lccr\ does not delay the evaluation of constraints, in the sense that the status of a constraint in state $s$ was calculated at the time of reaching $s$ through the evaluation of regression formulae.
%
%


\textbf{Monitoring Atoms.}~\lcc uses the same monitoring atoms as \lccr, plus one $prevent_\phi$ atom for each $\AO(\phi)$. Unlike \lccr, \lcc cannot foresee when an action will cause $\phi$ to hold for a second time, and thus uses $prevent_\phi$ to block further actions once this occurs.
%

\textbf{Initial State.}~In \lccr, to express the status of the constraints in the initial state,  we need to add to that state the monitoring atoms that hold initially.
This is not required in \lcc because of its delayed constraint evaluation, i.e., the status of a constraint in the initial state is expressed by the truth values of the monitoring atoms in its subsequent state.



\textbf{Actions.}~\lcc\ requires 1 additional action compared to \lccr, i.e., action \endaction, which is required to check whether the constraints are satisfied in a goal state, due to the 1-step constraint evaluation delay in \lcc.

\textbf{Complexity.}~While both \lccr and \lcc operate in polynomial time to the size of the problem (see Propositions \ref{prop:lccr_complexity} and \ref{prop:lcc_complexity}), \lcc runs in time that is \emph{linear} in the number of constraints and the number of actions, as because it determines all preconditions and effects required for the compilation with one pass over the constraints.
Instead, \lccr derives a set of preconditions and effects for each action of the domain, which more computationally involved.
%
%

\section{Experimental Evaluation}\label{sec:experiments}


\setlength{\tabcolsep}{1mm}
\begin{table}[t]
    \centering
    \small
    \begin{tabular}{@{$\;\;\;$}l@{$\;\;\;$}llccc}
    \toprule
    \multicolumn{2}{c}{Domain} & \lcc & LiftedTCORE & TCORE & LTL-C\\
    \midrule
    \multirow{7}{*}{\rotatebox{90}{Ground}} & Folding & \textbf{20} & \textbf{20} & \textbf{20} & \textbf{20} \\
    &Labyrinth & 17 & 17 & \textbf{18} & 16 \\
    &Quantum & \textbf{19} & \textbf{19} & 18 & 17 \\
    &Recharging & \textbf{19} & \textbf{19} & \textbf{19} & N.A. \\
    &Ricochet & \textbf{20} & \textbf{20} & \textbf{20} & \textbf{20} \\
    &Rubik's & \textbf{20} & \textbf{20} & \textbf{20} & N.A. \\
    &Slitherlink & \textbf{18} & \textbf{18} & 17 & \textbf{18} \\
    \addlinespace
    \multirow{7}{*}{\rotatebox{90}{Non-Ground}}&Folding & \textbf{18} & \textbf{18} & \textbf{18} & \textbf{18} \\
    &Labyrinth & 16 & 15 & \textbf{18} & 15 \\
    &Quantum & \textbf{19} & 17 & 18 & 16 \\
    &Recharging & 18 & 17 & \textbf{19} & N.A. \\
    &Ricochet & 12 & 11 & \textbf{14} & 7 \\
    &Rubik's & 16 & \textbf{17} & 14 & N.A. \\
    &Slitherlink & \textbf{14} & 10 & 12 & 11 \\
    \midrule
    &Total & \textbf{246} & 238 & 245 & 158 \\
    \bottomrule
    \end{tabular}
    \caption{Coverage achieved by all systems in all domains.}
    \label{tab:coverage}
\end{table}


Our aim is to evaluate \lccr and \lcc against state-of-the-art approaches. 
%
To do this, we generated a new benchmark based on the $7$ domains of the latest International Planning Competition (IPC)~\cite{DBLP:journals/aim/TaitlerAEBFGPSSSSS24}, including problems with many objects and high-arity actions that are challenging for grounding-based compilations.
%
%
Specifically, we created two datasets, each with $20$ tasks per domain: one with ground constraints, denoted by ``Ground", and one with constraints containing quantifiers ($\forall$, $\exists$), denoted by ``Non-Ground".  
These datasets contained $280$ problems in total and were generated as follows.
%
First, we used the IPC generators to construct problem instances without constraints.
Second, we introduced ground constraints that complicate an optimal solution of the problem using a custom constraint generator, yielding the Ground dataset.
Third, we modified the produced ground constraints by introducing quantifiers, leading to our Non-Ground dataset.
Additional details on constraint generation, along with statistics on the generated constraints, are provided in Appendix \ref{sec:appendix:experiments:constraints}.
%



We used TCORE as our baseline, as it has been proven more effective~\cite{DBLP:conf/aips/BonassiGPS21} than other methods~\cite{DBLP:conf/aips/Edelkamp06,DBLP:conf/aips/BentonCC12,DBLP:conf/ijcai/HsuWHC07,DBLP:conf/ijcai/TorresB15} for handling PDDL constraints. 
We also included LTL-C, a lifted compiler that supports quantified LTL constraints~\cite{DBLP:conf/aaai/BaierM06}. Unlike our encodings, LTL-C compiles the constraints by first translating them into finite-state automatons.
%
All compiled instances were solved using LAMA \cite{DBLP:journals/jair/RichterW10}, a state-of-the-art satisficing planner. 
We stopped LAMA after finding the first solution and disabled its invariant generation. All experiments were run on an Intel Xeon Gold 6140M 2.3 GHz, with runtime and memory limits of 1800s and 8GB, respectively.
%
%
\lccr\footnote{\url{https://github.com/Periklismant/LiftedTCORE}} and \lcc\footnote{\url{https://github.com/LBonassi95/NumericTCORE}} were implemented using the Unified Planning library~\cite{DBLP:journals/softx/MicheliBRSVFRTBGIIKPSSS25}; their code and our benchmark\footnote{\url{https://github.com/Periklismant/aaai26-pddl3-benchmark}} are publicly available.

\subsection{Experimental Results}\label{sec:results}

Table \ref{tab:coverage} presents the coverage (number of solved problems) per domain achieved by LAMA when planning over the compiled specification of each system,
We observe that TCORE, \lccr, and \lcc performed comparably. Overall, \lcc achieved the best performance, with TCORE being a close runner-up---these IPC domains were designed to be compatible with traditional grounding-based planners like LAMA, and thus this grounding-based planner performed well on tasks compiled by TCORE. Nonetheless, both \lcc and \lccr were competitive with TCORE and outperformed LTL-C, the state of the art among lifted compilation approaches. LTL-C failed in the `Recharging' and `Rubik's' domains, as it does not support universally quantified effects, which are present in these domains.

\begin{figure}[t]
\centering
  \begin{subfigure}[b]{0.48\columnwidth}
    \includegraphics[width=\textwidth]{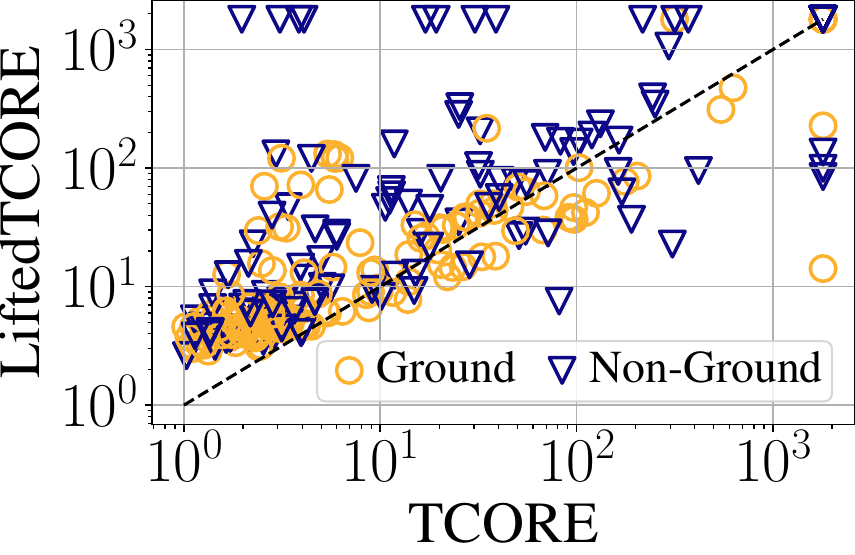}
    \label{fig:2}
  \end{subfigure}
  \begin{subfigure}[b]{0.48\columnwidth}
    \includegraphics[width=\textwidth]{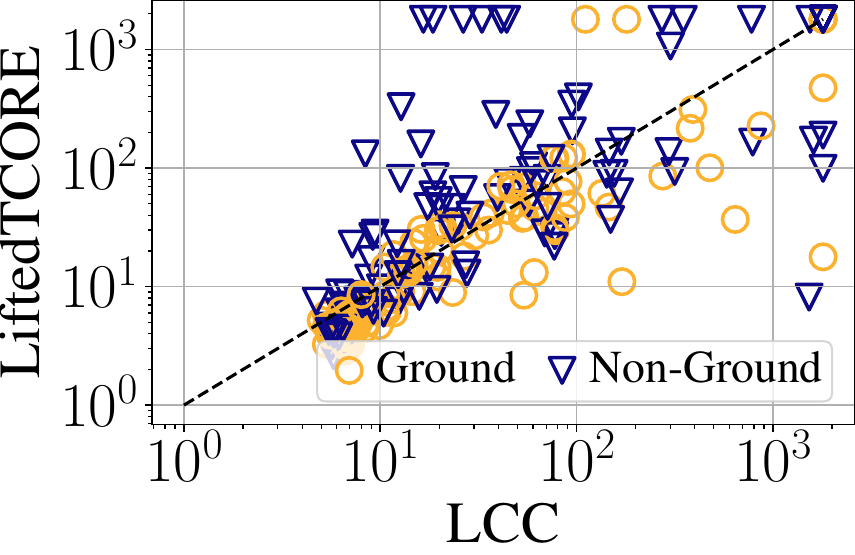}
    \label{fig:4}
  \end{subfigure}
\caption{Runtime comparison between \lccr vs. TCORE and between \lccr vs. \lcc.}
\label{fig:scatter}
\end{figure}

\par In terms of coverage, \lcc outperformed \lccr; \lccr's lifted regression operator introduced complex precondition and effect formulas, especially in cases with Non-Ground constraints, hindering LAMA’s performance.
On the other hand, this operator allows constraint evaluation without the 1-step delay of \lcc, allowing LAMA to perform a more efficient search. Table \ref{tab:expanded} reports the average number of nodes expanded by LAMA when planning over the compiled specifications of \lcc, \lccr and TCORE. On average, \lccr led to fewer node expansions than \lcc. TCORE and \lccr\ led to a similar number of node expansions, as they both employ regression to foresee redundant expansions.
\par The reduction in expansion nodes also affects runtime; Figure \ref{fig:scatter} shows that \lccr led to more efficient planning than \lcc in several instances. 
Specifically, constraint compilation with \lccr led to more efficient planning than \lcc compilation in 161 out of the 234 instances that were solved by both approaches.
\lccr performed better on instances with Ground constraints, where both approaches introduces formulas with comparable complexity.
In the Non-Ground case, the size of the largest formula introduced by \lccr\ and \lcc per instance was, on average, 26 and 18 atoms, thus leading to slower planning in the case of \lccr.
Our results show that \lccr compilation led to slower planning than TCORE compilation, although there is complementarity. 

\par We conclude with statistics on the size of the compiled problems and their compilation times, reported in Table \ref{tab:compilation_statistics}. On average, TCORE produced problems with thousands of actions and effects, reaching up to 411K actions and 1.8M effects, while both \lccr and \lcc produced problems that are more succinct by orders of magnitude. This not only makes the compiled problems more understandable and easier to debug, but also facilitates the use of planners that do not need to perform grounding. Currently, the support of complex preconditions and conditional effects is limited in these planners \cite{DBLP:conf/aaai/HorcikFT22,DBLP:conf/ecai/HorcikF23,DBLP:conf/ijcai/CorreaG24}. We expect that the benefits of our lifted compilers will become more prominent as more expressive heuristics for lifted planning are developed. 

\textbf{Limitations:} We observed that lifted constraint compilers may not be beneficial for problems with a large number of constraints. 
Lifted methods update all (lifted) actions that may affect the corresponding constraints in a least one of their instantiations, resulting in numerous new action preconditions and effects that hinder the performance of modern planners like LAMA.
In contrast, TCORE updates ground actions sparsely, as only few of them may affect each constraint.
We evaluated our compilers on a benchmark including problems with hundreds of constraints~\cite{DBLP:conf/aips/BonassiGPS21}, and TCORE proved to be the best performing compiler (see Appendix \ref{sec:appendix:experiments:icaps21}).
We believe that further research on lifted planning heuristics could alleviate this issue, while preserving the benefits of lifted compilations, such as compilation efficiency and compiled specification succinctness.

\begin{table}[t]
    \centering
    \small
\begin{tabular}{@{$\;\;\;$}l@{$\;\;\;$}lccc}
\toprule
\multicolumn{2}{c}{Domain} & \lcc & \lccr & TCORE \\
\midrule
\multirow{7}{*}{\rotatebox{90}{Ground}} & Folding & 150.45 & \textbf{132.10} & 139.55 \\
&Labyrinth & \textbf{769.69} & 1453.44 & 1464.94 \\
&Quantum & 338.24 & 379.82 & \textbf{307.59} \\
&Recharging & 9259.16 & \textbf{1629.53} & 1796.47 \\
&Ricochet & 2412.70 & 2153.00 & \textbf{1956.75} \\
&Rubik's & 230268.65 & \textbf{180999.25} & 189079.45 \\
&Slitherlink & 19483.47 & 13454.12 & \textbf{11437.53} \\
\addlinespace
\multirow{7}{*}{\rotatebox{90}{Non-Ground}}&Folding & 43.89 & \textbf{37.61} & 40.33 \\
&Labyrinth & \textbf{623.57} & 3028.00 & 3230.93 \\
&Quantum & 408214.12 & 224.18 & \textbf{216.18} \\
&Recharging & 1305.18 & \textbf{459.24} & 673.18 \\
&Ricochet & 327.73 & \textbf{121.27} & 156.73 \\
&Rubik's & 782161.29 & \textbf{316887.57} & 382046.86 \\
&Slitherlink & \textbf{478.40} & 964.00 & 780.40 \\
\bottomrule
\end{tabular}
\caption{Average number of nodes expanded by LAMA on tasks compiled by TCORE, \lccr, and \lcc. Averages are computed among instances solved by all systems.}
\label{tab:expanded}
\end{table}

\begin{table}[t]
    \centering
       \small 
    \begin{tabular}{@{}l@{$\;\;$}c@{$\;\;$}c@{$\;\;$}c@{$\;\;$}c@{$\;\;$}c@{$\;\;$}c@{}} 
    \centering
    \multirow{2}{*}{} & \multicolumn{3}{c}{Ground} & \multicolumn{3}{c}{Non-Ground} \\
    & \makecell{Lifted\\TCORE} & \lcc & TCORE & \makecell{Lifted\\TCORE} & \lcc & TCORE   \\
    \cmidrule(r){2-4}
    \cmidrule(r){5-7}
    Actions & 7 & 8 & 21K & 7 & 8 & 21K\\
    Effects & 58 & 63 & 95K & 60 & 67 & 96K \\
    Comp. Time & 0.023 & 0.003 & 4.518 & 0.081 & 0.003 & 5.956
    \end{tabular}
    \caption{Average number of actions, effects and compilation time (in seconds) on the problems in our dataset.}
    \label{tab:compilation_statistics}
\end{table}

\section{Summary and Further Work}\label{sec:summary}

We proposed two methods, \lccr and \lcc, for compiling away quantitative state-trajectory constraints from a planning problem without grounding it.
We studied both compilers theoretically, proving their correctness and deducing their worst-case time complexity, and qualitatively compared their key features.
Moreover, we presented an empirical evaluation of our methods on the domains included in the latest International Planning Competition, demonstrating that our compilers are efficient and lead to significantly more succinct compiled specifications compared to a state-of-the-art compiler that grounds the domain, while yielding competitive performance when used for planning with LAMA.

In the future, we aim to design lifted compilers for numeric and metric time constraints~\cite{DBLP:conf/aaai/BonassiGS24}.

\section*{Acknowledgements}
Periklis Mantenoglou and Pedro Zuidberg Dos Martires were supported by the Wallenberg AI, Autonomous Systems and Software Program (WASP) funded by the Knut and Alice Wallenberg Foundation. Luigi Bonassi was supported by the joint UKRI and AISI-DSIT Systemic Safety Grant [grant number UKRI854].

\bibliography{aaai2026}

\clearpage


\appendix

\section{Proofs for Section~\ref{sec:lccr}}\label{proofs_lccr}

\subsection{Proof of Lemma~\ref{prop:lifted_regr_correctness}}
\label{proof:prop:lifted_regr_correctness}
\propliftedregrcorrectness*

\begin{proof}
We use an inductive proof on the structure of formula $\phi$.
Note that, although $\phi$ is a closed formula, a subformula $\psi$ of $\phi$ may contain variables that are quantified in $\phi$ outside the scope of $\psi$.
We prove that, for all possible substitutions $\substpsi$ of the variables in $\psi$ with domain objects, it holds that:
\begin{align}\label{eq:fcase}
    s\models \Lregression(\psi,a)\applys{(\substainst \cup \substpsi)} \iff s[\actioninst]\models \psi\applys{\substpsi}
\end{align}
Henceforth, we use $\substpsia$ as a shorthand for $\substainst\cup\substpsi$.

We start with the base case where $\psi$ is an atom $f$.
Based on Definition \ref{def:lifted_regression}, we have:
\begin{align*}\label{eq:atom_lregression}
s\models\Lregression(f,a)\applys{\substfa}\iff &s\models \LGamma_{f}(a)\applys{\substfa}\vee \\
&(s\models f\applys{\substf}\wedge s\not\models\LGamma_{\neg f}(a)\applys{\substfa})
\end{align*}
Since $s[\actioninst]\val \bigcup_{\substack{\cond{\groundcond}{e^{+}}\in \actionEff{\actioninst}\\ s\models \groundcond}} e^{+} {\cup} (s{\setminus} \bigcup_{\substack{\cond{\groundcond}{e^-}\in \actionEff{\actioninst}\\ s\models \groundcond}} e^{-})$, in order for $f\applys{\substf}$ to be included in $s[\actioninst]$, there needs to be an effect of ground action $\actioninst$ that brings about $f\applys{\substf}$, or for $f\applys{\substf}$ to persist from the previous state $s$ while no effect of $\actioninst$ brings about $\neg f\applys{\substf}$.

Case 1: $f\applys{\substf}\in \bigcup_{\substack{\cond{\groundcond}{e^{+}}\in \actionEff{\actioninst}: s\models \groundcond}} e^{+}$. We have:
\begin{align*}
&s\models \LGamma_{f}(a)\applys{\substfa} \xLeftrightarrow{Def.~\ref{def:lifted_gamma}}\\
&\exists (\condz{z}{c}{e}){\in}\actionEff{a}, \exists \xi(f,e):\\
&\quad s\models (\exists\zfree{:} c\applys{\substz(f,e)}\wedge \bigwedge_{\mathclap{(t_i\doteq u_i)\in \xi(l,e)\setminus\xi_z(l, e)}} t_i=u_i)\applys{\substfa} \xLeftrightarrow{(\alpha)} \\
&\exists \cond{\groundcond}{f\applys{\substf}}\in \actionEff{\actioninst}: s\models \groundcond \iff f\applys{\substf}\in \bigcup_{\substack{\mathclap{\cond{\groundcond}{e^{+}}\in \actionEff{\actioninst}: s\models \groundcond}}} e^{+}
\end{align*}
It remains to prove equivalence ($\alpha$), which relates the conditions under which actions $a$ and $\actioninst$ may bring about $f\applys{\substf}$.
We prove both directions by contradiction.
We use the fact that, for each effect $\condz{z}{c}{e}$ of $a$, $\actioninst$ contains one effect for each possible substitution $\substzi{i}$ of the variables in $\vec{z}$ with domain objects.
If there are $m$ such substitutions, then $\actioninst$ contains effects $\cond{c\applys{\substza{i}}}{e\applys{\substza{i}}}$, $1\leq i\leq m$, where $\substza{i}\val \substzi{i}\cup \substainst$.

Case $\Rightarrow$:
Suppose that the left-hand side of equivalence ($\alpha$) holds, but there is no satisfied effect of $\actioninst$ that brings about $f\applys{\substf}$.
Therefore, for each effect $\cond{c\applys{\substza{i}}}{e\applys{\substza{i}}}$ of $\actioninst$, we have $s\not\models c\applys{\substza{i}}$ or that $e\applys{\substza{i}}$ and $f\applys{\substf}$ are not unifiable. 
Suppose that, for all such effects, $s\not\models c\applys{\substza{i}}$.
Since there is an effect $\cond{c\applys{\substza{i}}}{e\applys{\substza{i}}}$ in $\actioninst$ for all possible substitutions of the variables in $\vec{z}$, this implies that there is no substitution $\theta$ such that $s\models (c\applys{\substainst})\applys{\theta}$.
However, according to the left-hand side of equivalence $(\alpha)$, under substitutions $\theta_z(f,e)$ and $\substfa$, there is an assignment $\substzfree$ to the free variables in $(c\applys{\theta_z(f,e)})\applys{\substfa}$, i.e., variables $\zfree$, such that $s\models (c\applys{\theta_z(f,e)\cup \substzfree})\applys{\substfa}$.
Therefore, $s\models (c\applys{\substainst})\applys{\theta}$ is true when $\theta$ we get by applying  $\theta_z(f,e)\cup \substzfree$ first and then $\substf$, which is a contradiction.

Consider effect $\cond{(c\applys{\theta_z(f,e)\cup \substzfree})\applys{\substfa}}{(e\applys{\theta_z(f,e)\cup \substzfree})\applys{\substfa}}$ of $\actioninst$ and suppose that $(e\applys{\theta_z(f,e)\cup \substzfree})\applys{\substfa}$ and $f\applys{\substf}$ are not unifiable.
However, according to the left-hand side of equivalence $(\alpha)$, there is a most general unifier $\xi(f,e)$ between $f$ and $e$, which, due the satisfaction of equalities $(t_i\val u_i)\applys{\substfa}$, successfully unifying the arguments of $f\applys{\substf}$ with the appropriate arguments of $\actioninst$, enforces substitution $\theta_z(f,e)\cup \substainst$ onto $e$, while no variable in $\zfree$ appears in $e$.
As a result, $(e\applys{\theta_z(f,e)\cup \substzfree})\applys{\substfa}$ and $f\applys{\substf}$ can, in fact, be unified, which is a contradiction.

Both cases led to a contradiction, thus proving that, if the left-hand side of equivalence ($\alpha$) holds, then there is a satisfied effect of $\actioninst$ that brings about $f\applys{\substf}$.

Case $\Leftarrow$:
Suppose that there is a satisfied effect $\cond{\groundcond}{f\applys{\substf}}$ of action $\actioninst$ in state $s$, but the left-hand side of equivalence $(\alpha)$ does not hold.
In this case, for every effect $\condz{z}{c}{e}$ in action $a$, either $e$ cannot be unified with $f\applys{\substf}$, or there is no substitution for which $c$ is satisfied in $s$.
However, since action $\actioninst$ is a ground incarnation of action $a$, there is an effect $\condz{z}{c}{e}$ of $a$ whose grounding via the arguments of $\actioninst$ led to effect $\cond{\groundcond}{f\applys{\substf}}$ in $\actioninst$.
This grounding constitutes a substitution $\theta$, which is the union of substitution $\substainst$ with the substitution $\theta_z$ for the variables in $\vec{z}$ that unifies $c$ with $\groundcond$ and $e$ with $f\applys{\substf}$, which is a contradiction. 
Thus, if there is a satisfied effect $\cond{\groundcond}{f\applys{\substf}}$ of action $\actioninst$ in state $s$, then the left-hand side of equivalence $(\alpha)$ holds.
%
%
%
%

Case 2: $f\applys{\substf}\in s \wedge \neg f\applys{\substf}\not\in\bigcup_{\substack{\cond{c}{e^{-}}\in \actionEff{\actioninst}: s\models c}} e^{-}$.
It suffices to prove that $s\models \LGamma_{\neg f}(a)\applys{\substfa} \iff \neg f\applys{\substf}\in \bigcup_{\substack{\cond{\groundcond}{e^{-}}\in \actionEff{\actioninst}: s\models \groundcond}} e^{-}$ which is possible by following the steps in case 1.

By combining the results of the two cases, we prove equivalence \eqref{eq:fcase} when $\psi$ is an atom $f$, i.e., the base case of our inductive proof.

For the inductive step, we assume that equivalence \eqref{eq:fcase} holds for first-order formulae $\phi$, $\phi_1$ and $\phi_2$, and prove that equivalence \eqref{eq:fcase} holds for formulae $\neg \phi$, $\phi_1\wedge \phi_2$, $\phi_1\vee \phi_2$, $\forall\vec{x}\phi$ and $\exists\vec{x}\phi$.
We leverage the following properties of the regression operator, which hold because operator $\Lregression$ only induces changes in $\phi$ at the atom level, while preserving the structure of $\phi$ (see Definition \ref{def:lifted_regression}). 
\begin{align*}
\Lregression(\neg\phi, a) &\iff \neg \Lregression(\phi, a) \\
\Lregression(\phi_1\wedge \phi_2, a) &\iff \Lregression(\phi_1, a) \wedge \Lregression(\phi_2, a) \\
\Lregression(\phi_1\vee \phi_2, a) &\iff \Lregression(\phi_1, a) \vee \Lregression(\phi_2, a) \\
\Lregression(\forall\vec{x}\phi, a) &\iff \forall \vec{x} \Lregression(\phi, a) \\
\Lregression(\exists\vec{x}\phi, a) &\iff \exists \vec{x} \Lregression(\phi, a) 
\end{align*}
For the case of $\neg \phi$, we prove equivalence \eqref{eq:fcase} as follows:
\begin{align*}
&s\models \Lregression(\neg \phi, a)\applys{\substainst} \iff s\models \neg \Lregression(\psi, a)\applys{\substainst}\iff \\
&s\not\models \Lregression(\psi, a)\applys{\substainst}\iff s[\actioninst]\not\models \psi \iff s[\actioninst]\models \neg \psi
\end{align*}
The proofs for the remaining cases are similar.
\end{proof}

\subsection{Proof of Proposition~\ref{prop:lccr_correctness}}
\label{proof:prop:lccr_correctness}

\proplccrcorrectness*

\begin{proof}
Let:
\begin{compactitem} 
\item $\pddlProblem\val \tup{F, O, A, I, G, C}$
\item $\pddlProblem'\val \tup{F', O, A', I', G', \emptyset}$
\item $\plan\val \tup{\actioninst_0, \dots, \actioninst_{n\minus 1}}$ 
\end{compactitem}
Additionally, $\sigma\val \tup{s_0, \dots, s_n}$ and $\sigma'\val \tup{s'_0, \dots, s'_n}$ denote the state trajectories induced by executing plans $\plan$ on problems $\pddlProblem$ and $\pddlProblem'$, respectively, while $Pre'(a)$ and $Eff'(a)$ denote the preconditions and the effects of an action $a$ in the compiled problem $\pddlProblem'$.

We use the following deductions in our proof, which follow directly from Algorithm \ref{alg:lifted_tcore}. 
\begin{compactenum}
\item Action $\actioninst_i$ has strictly more restrictive preconditions in $A'$ than in $A$, i.e., $Pre'(\actioninst_i)\models \actionPrec{\actioninst_i}$.
\item The effects of $\actioninst_i$ in $A'$ that are different from the ones of $\actioninst_i$ in $A$ affect only monitoring atoms, i.e., $hold_c$ and $seen_\phi$, and thus, for each formula $\phi$ without monitoring atoms, we have $s_i\models \phi$ iff $s'_i\models \phi$.
\end{compactenum}

Suppose that $\plan$ is solution for $\pddlProblem$, but it is not a solution for $\pddlProblem'$.
Since $\plan$ is not a solution for problem $\pddlProblem'$, and $\pddlProblem'$ does not have constraints, then one of the following must hold: 
\begin{compactenum}
\item $\exists \actioninst_i, 0\leq i < n: s'_i\not\models Pre'(\actioninst_i)$
\item $s'_n\not\models G'$
\end{compactenum}
In case 1, where state $s'_i$ does not satisfy the precondition of action $\actioninst_i$, since $\plan$ is a solution for $\pddlProblem$, we have that $s_i\models \actionPrec{\actioninst_i}$.
Since $s_i$ and $s'_i$ differ only on monitoring atoms, we have $s'_i\models \actionPrec{\actioninst_i}$.
Moreover, because $Pre'(\actioninst_i)\val \actionPrec{\actioninst_i}\wedge p_1 \wedge \dots \wedge p_k$ and $s'_i\not\models Pre'(\actioninst_i)$, there is some $p_c\in \{p_1, \dots, p_k\}$ such that $s'_i\not\models p_c$.
Based on Algorithm \ref{alg:lifted_tcore}, precondition $p_c$ is introduced by the compilation of either an $\A(\phi)$, an $\AO(\phi)$ or a $\SB(\phi, \psi)$ constraints.
\begin{compactitem}
    \item Case $\A(\phi)$. $p_c$ is $\Lregression(\phi, a)\applyai{i}$, where $\subst_i$ substitutes the parameters of $a$ with the ground arguments of $\actioninst_i$. 
    We have $s'_i\not\models\Lregression(\phi, a)\applyai{i}$, which, since $\Lregression(\phi, a)\applyai{i}$ does not contain monitoring atoms, implies that $s_i\not\models\Lregression(\phi, a)\applyai{i}$.
    According to Lemma \ref{prop:lifted_regr_correctness}, $s_i\not\models\Lregression(\phi, a)\applyai{i}$ implies that $s_i[\actioninst_i]\not\models \phi$.
    Thus, there is a state induced by executing plan $\plan$ in which $\phi$ does not hold.
    This is contraction because then plan $\plan$ violates constraint $\A(\phi)$ in $\pddlProblem$.
    \item Case $\AO(\phi)$. $p_c$ is $\neg (seen_\phi \wedge \neg\phi \wedge \Lregression(\phi, a)\applyai{i})$, and thus $s'_i\models seen_\phi \wedge \neg\phi\wedge \Lregression(\phi, a)\applyai{i}$. 
    According to effect $\cond{\Lregression(\phi, a)\applyai{i}}{seen_\phi}$ of $Eff'(\actioninst_i)$ and Lemma \ref{prop:lifted_regr_correctness}, atom $seen_\phi$ is true in $s'_i$ iff there is a state before or at $s'_i$ where $\phi$ holds. 
    Moreover, condition $\neg\phi \wedge \Lregression(\phi, a)\applyai{i}$ expresses that $\phi$ is false at $s'_i$ and it is true at $s'_i[\actioninst_i]$.
    Since $\phi$ does not contain monitoring atoms, the above implies that $\plan$ induces states $s_k$, $s_i$ and $s_i[\actioninst_i]$ in $\pddlProblem$, where $k<i$, such that $\phi$ is true in $s_k$, $\phi$ is false in $s_i$ and $\phi$ is true in $s_i[\actioninst_i]$.
    This violates constraint $\AO(\phi)$ of $\pddlProblem$ leading to a contradiction.
    \item Case $\SB(\phi, \psi)$. $p_c$ is $\Lregression(\phi, a)\applyai{i}\rightarrow seen_\psi$, and thus $s'_i\models \Lregression(\phi, a)\applyai{i}\wedge\neg seen_\psi$. The monitoring atom $seen_\psi$ is true iff there exists a state $s'_j$, $j\leq i$, where $\psi$ is true (see line \ref{line:lccr_case_sb_e} of Algorithm \ref{alg:lifted_tcore}).
    Since $\phi$ and $\psi$ do not contain monitoring atoms, we have $s_i\models \Lregression(\phi, a)\applyai{i}$ and that there is no state $s_j$, $j\leq i$, where $\psi$ is true.
    Thus, $\phi$ holds in state $s_i[\actioninst_i]$ and there is no earlier state where $\psi$ holds, which is a contradiction because it violates constraint $\SB(\phi, \psi)$.
\end{compactitem}

Case 2 expresses that the final state $s'_n$ induced by plan $\plan$ on problem $\pddlProblem'$ does not satisfy its goal $G'$.
Based on Algorithm \ref{alg:lifted_tcore}, $G'$ is strictly more restrictive that goal $G$, as it additionally requires that all the $hold_c$ monitoring atoms are true.
Moreover, we have $s_n\models G$, because $\plan$ is a solution for $\pddlProblem$, which implies that $s'_n\models G$, as $G$ does not contain monitoring atoms.
As a result, $s'_n$ must be violating a $hold_c$ atom.
Thus, there is a constraint $c\val \ST(\phi)$ or $c\val \SA(\phi, \psi)$ in $\pddlProblem$ such that $s'_n\not\models hold_c$.
\begin{compactitem}
    \item Case $c$ is $\ST(\phi)$. Since $hold_c$ may only come about via effect $\cond{\Lregression(\phi, a)}{hold_c}$ and $s'_n\not\models hold_c$, then there is no action $\actioninst_i$ in $\plan$ that brings about $\phi$ when executed on a state $s'_i$ of $\sigma'$. 
    Since $\Lregression(\phi, a)$ does not include monitoring atoms, this implies that there neither a state $s_i$ in $\sigma$ such that action $\actioninst_i$ brings about $\phi$ in $s_i[\actioninst_i]$.
    This is a contradiction because $\plan$ solves $\pddlProblem$ and $\ST(\phi)$ is one of its constraints.
    \item Case $c$ is $\SA(\phi, \psi)$. Based on the effects manipulating $hold_c$, we have $s'_n\not\models hold_c$ iff for each state $s'_j$ such that $\actioninst_j$ leads to a state $s'_j[\actioninst_j]$ where $\phi$ is true and $\psi$ is false, there is a subsequent state $s'_i$, $i>j$ where applying $\actioninst_i$ leads to a state where $\psi$ holds.
    Moving this situation to the trace $\sigma$ induced by $\plan$ in $\pddlProblem$, we deduce that $\plan$ violates constraint $\SA(\phi, \psi)$ of $\pddlProblem$, which is a contradiction because $\plan$ is a solution of $\pddlProblem$.
\end{compactitem}

By contradicting cases 1 and 2, we have shown that if $\plan$ is solution for $\pddlProblem$ then plan $\plan$ is a solution for $\pddlProblem'$.

Suppose that $\plan$ is a solution for $\pddlProblem'$, but not a solution for $\pddlProblem$.
One of the following must hold:
\begin{compactenum}
\item $\exists \actioninst_i, 0\leq i < n: s_i\not\models \actionPrec{\actioninst_i}$
\item $s_n\not\models G$
\item $\exists c\in C$ such that $\sigma$ does not satisfy $c$.
\end{compactenum}
Case 1 cannot be true as $\forall \actioninst_i, 0\leq i < n: s'_i\models Pre'(\actioninst_i)$, since $\plan'$ is a solution for $\pddlProblem'$, $Pre'(\actioninst_i)\models \actionPrec{\actioninst_i}$, and $s'_i\models \actionPrec{\actioninst_i}$ iff $s_i\models\actionPrec{\actioninst_i}$.
Case 2 cannot be true because we have $s'_n\models G'$, which is implies $s_n\models G$.

We treat case 3 depending on the type of constraint $c$.
\begin{compactitem}
    \item Case $\A(\phi)$. There is a state $s_i$ such that $s_i\models \neg\phi$, which implies that $s'_i\models \neg\phi$.
    However, since, for every action $a$, precondition $Pre'(a)$ requires that $\Lregression(\phi, a)$ holds, i.e., $a$ does not bring about $\neg\phi$ in the next state, action $\actioninst_{i-1}$ cannot be applied on state $s'_{i\minus 1}$.
    Thus, plan $\plan$ is not executable on $\pddlProblem'$, which is a contradiction.
    \item Case $\ST(\phi)$. There is no state in $\sigma$ where $\phi$ holds. As a result, there is neither such a state in $\sigma'$, meaning that effect $\cond{\Lregression(\phi, a)\applyai{i}}{hold_c}$ is not triggered for any action $\actioninst_i$ in $\plan$.
    Thus, $hold_c$ does not hold at $s'_n$, which is a contradiction as $hold_c$ is required in the goal of $\pddlProblem'$.
    \item Case $\AO(\phi)$. There are states $s_l$, $s_j$ and $s_i$ such that $l<j<i$, $\phi$ is true in $s_l$, $\phi$ is false in $s_j$ and $\phi$ is true in $s_i$.
    This implies that $\phi$ has the corresponding truth values in states $s'_l$, $s'_j$ and $s'_i$.
    Based on the action effects in $\pddlProblem'$, $seen_\phi$ holds at state $s'_l$.
    Suppose that $s'_k$, $j\leq k< i$, is the last state before $s'_i$ where $\phi$ is false.
    Then, $seen_\phi\wedge\neg\phi\wedge \Lregression(\phi, a)\applyai{k}$ is satisfied in $s'_k$, meaning that action $a_k$ is not applicable in $s'_k$, which is a contradiction because plan $\plan$ is a solution for $\pddlProblem'$.
    \item Case $\SB(\phi,\psi)$. There is a state $s_i$ such that $\phi$ holds in $s_i[\actioninst_i]$ and there is no state $s_j$, $j\leq i$, where $\psi$ holds.
    This implies that $s'_i[\actioninst_i]\models \phi$ and $\forall j\leq i: s'_i\models \neg \psi$.
    As a result, action effect $\cond{\Lregression(\psi, a)}{seen_\psi}$ is never satisfied before $s'_i[\actioninst_i]$, meaning that $seen_\psi$ is false in $s'_i[\actioninst_i]$.
    Therefore, based on precondition $\phi\rightarrow seen_\psi$, action $\actioninst_i$ is not applicable in state $s'_i$, which is a contradiction.
    \item Case $\SA(\phi,\psi)$. There is a state $s_j$ where $\phi\wedge \neg\psi$ holds and there is no state after $s_i$, $i>j$, where $\psi$ holds.
    This implies that $s'_j\models \phi\wedge\neg\psi$ and $\forall j<i\leq n: s'_i\models \neg\psi$.
    Then, based on the effects in lines \ref{line:lccr_case_sa}--\ref{line:lccr_case_sa_e} of Algorithm \ref{alg:lifted_tcore}, $hold_c$ does not hold in the final state $s'_n$, which is a contradiction because $hold_c$ is required by goal $G'$.
\end{compactitem}
By contradiction, we have shown that if plan $\plan$ is a solution for $\pddlProblem'$ then it is also a solution for $\pddlProblem$.

Therefore, if \lccr compiles a problem $\pddlProblem$ into problem $\pddlProblem'$, then a plan $\plan$ is valid for $\pddlProblem$ iff $\plan$ is valid for $\pddlProblem'$.
\end{proof}

\subsection{Proof of Proposition \ref{prop:lccr_complexity}} \label{proof:prop:lccr_complexity}

\proplccrcomplexity*

\begin{proof}
The first step of \lccr is to check whether a constraint is violated in the initial state of the problem, rendering it unsolvable (see lines \ref{line:lccr_check_unsolvable}--\ref{line:lccr_unsolvable} of Algorithm \ref{alg:lifted_tcore}).
To do this, \lccr iterates over the constraints in $C$ and, in the worst-case, solves a model checking $I\models \phi$ problem for each constraint. 
The cost of this model checking problem is $\mathcal{O}(n_f^b)$~\cite{DBLP:conf/stoc/Vardi82}, and thus the cost of this step is $\mathcal{O}(n_cn_f^b)$.
Second, \lccr introduces the necessary monitoring atoms for the constraints in $C$, which is done in $n_c$ steps (line \ref{line:lccr-F'}).
Third, \lccr identifies the monitoring atoms that need to be added in the initial state (line \ref{line:lccr-I'}).
The cost of identifying these atoms is $\mathcal{O}(n_cn_f^b)$, as it requires solving model checking problems, similarly to the first step.
Fourth, \lccr iterates over each action-constraint pair in $\pddlProblem$ and computes the preconditions and effects that need to be added to the action (lines \ref{line:lccr_foraction}--\ref{line:lccr_addeff} and lines \ref{line:lccr_initpe}--\ref{line:lccr_returnpe}).
Identifying these preconditions and effects requires, in the worst case, the evaluation of two lifted regression expressions. 
To evaluate $\Lregression(\phi, a)$ for a formula $\phi$ and action $a$, we need to compute $\Lregression(f, a)$ for each atom $f$ in $\phi$, whose number is bounded by $n_f$.
$\Lregression(f, a)$ requires the computation of $\LGamma_f(a)$ and $\LGamma_{\neg f}(a)$. 
To derive $\LGamma_l(a)$, we need to compute, for each effect $e$ of action $a$---whose number is bounded by $n_f$---the most general unifier $\xi(l,e)$, which requires $n_k$ steps, i.e., one iteration over the argument of $l$ and $e$.
Therefore, the overall cost of computing the preconditions and effects that need to be added to the actions of the problem is $\mathcal{O}(n_an_cn_f^2n_k)$.
Fifth, \lccr added to the goal state all $hold_c$ monitoring atoms, which requires at most $n_c$ steps (line \ref{line:lccr_return}).

Based on the above steps, the worst-case time complexity of \lccr is $\mathcal{O}(2n_cn_f^b\plus 2n_c\plus n_a n_c n^2_f n_k)$, which, after simplifications, becomes $\mathcal{O}(n_cn_f^b \plus n_a n_c n^2_f n_k)$.
\end{proof}

\section{Proofs for Section~\ref{sec:lcc}}\label{proofs_lcc}

\subsection{Proof of Proposition \ref{prop:lcc_correctness}}
\label{proof:prop:lcc_correctness}

\proplcccorrectness*

\begin{proof}
Let:
\begin{compactitem} 
\item $\pddlProblem\val \tup{F, O, A, I, G, C}$
\item $\pddlProblem'\val \tup{F', O, A', I', G', \emptyset}$
\item $\plan\val \tup{\actioninst_0, \dots, \actioninst_{n\minus 1}}$ 
\item $\plan'\val \tup{\actioninst_0, \dots, \actioninst_{n\minus 1}, \endaction}$.
\end{compactitem}
Additionally, $\sigma\val \tup{s_0, \dots, s_n}$ and $\sigma'\val \tup{s'_0, \dots, s'_n, s'_f}$ denote the state trajectories induced by executing plans $\plan$ and $\plan'$ on problems $\pddlProblem$ and $\pddlProblem'$, respectively, while $Pre'(a)$ and $Eff'(a)$ denote the preconditions and the effects of an action $a$ in the compiled problem $\pddlProblem'$.

We use the following deductions in our proof, which follow directly from Algorithm \ref{alg:lifted_no_regr}. 
\begin{compactenum}
\item Action $\actioninst_i$ has strictly more restrictive preconditions in $A'$ than in $A$, i.e., $Pre'(\actioninst_i)\models \actionPrec{\actioninst_i}$.
\item The effects of $\actioninst_i$ in $A'$ that are different from the ones of $a_i$ in $A$ affect only monitoring atoms, i.e., $hold_c$, $seen_\phi$ and $prevent_\phi$ atoms, and thus, for each formula $\phi$ without monitoring atoms, we have $s_i\models \phi$ iff $s'_i\models \phi$.
\end{compactenum}

Suppose that $\plan$ is solution for $\pddlProblem$, but $\plan'$ is not a solution for $\pddlProblem'$.
Since $\plan'$ is not a solution for problem $\pddlProblem'$, and $\pddlProblem'$ does not have constraints, then one of the following must hold: 
\begin{compactenum}
\item $\exists \actioninst_i, 0\leq i < n: s'_i\not\models Pre'(\actioninst_i)$
\item $s'_n\not\models Pre'(\endaction)$
\item $s'_f\not\models G'$
\end{compactenum}
In case 1, where state $s'_i$ does not satisfy the precondition of action $\actioninst_i$, since $\plan$ is a solution for $\pddlProblem$, we have that $s_i\models \actionPrec{\actioninst_i}$.
Since $s_i$ and $s'_i$ differ only on monitoring atoms, we have $s'_i\models \actionPrec{\actioninst_i}$.
Moreover, because $Pre'(\actioninst_i)\val \actionPrec{\actioninst_i}\wedge p_1 \wedge \dots \wedge p_k$ and $s'_i\not\models Pre'(\actioninst_i)$, there is some $p_c\in \{p_1, \dots, p_k\}$ such that $s'_i\not\models p_c$.
Based on Algorithm \ref{alg:lifted_no_regr}, precondition $p_c$ is introduced by the compilation of either an $\A(\phi)$, an $\AO(\phi)$ or a $\SB(\phi, \psi)$ constraints.
\begin{compactitem}
    \item Case $\A(\phi)$. $p_c$ is $\phi$, and thus $s'_i\not\models\phi$, which, since $\phi$ does not contain monitoring atoms, implies that $s_i\not\models\phi$. This is a contradiction because $s_i$ is a state induced by executing solution $\plan$ of $\pddlProblem$ and $\pddlProblem$ has constraint $\A(\phi)$.
    \item Case $\AO(\phi)$. $p_c$ is $\neg (\phi\wedge prevent_\phi)$, and thus $s'_i\models\phi\wedge prevent_\phi$. According to the effects of the actions in $A'$, the monitoring atom $prevent_\phi$ is true iff there is a state $s'_j$, $j<i$, where $\phi$ is false and atom $seen_\phi$ is true (see line \ref{line:lcc_amo_eff1} of Algorithm \ref{alg:lifted_no_regr}). In turn, $seen_\phi$ is true iff there is a state $s'_l$, $l<j$, where $\phi$ is true (line \ref{line:lcc_amo_eff2}).
    Since $\phi$ does not contain monitoring atoms, the above implies that the states $s_l$, $s_j$ and $s_i$ of solution $\plan$ of $\pddlProblem$ are such that $\phi$ is true in $s_l$, $\phi$ is false in $s_j$ and $\phi$ is true in $s_i$.
    This violates constraint $\AO(\phi)$ of $\pddlProblem$ leading to a contradiction.
    \item Case $\SB(\phi, \psi)$. $p_c$ is $\phi\rightarrow seen_\psi$, and thus $s'_i\models \phi\wedge\neg seen_\psi$. The monitoring atom $seen_\psi$ is true iff there exists a state $s'_j$, $j<i$, where $\psi$ is true (line \ref{line:lcc_sb_eff}).
    Since $\phi$ and $\psi$ do not contain monitoring atoms, we have $s_i\models \phi$ and that there is no state $s_j$, $j<i$, preceding $s_i$ where $\psi$ is true.
    This is a contradiction because $\plan$ is a solution of $\pddlProblem$ and $\pddlProblem$ includes constraint $\SB(\phi, \psi)$.
\end{compactitem}

Case 2 expresses that state $s'_n$ does not satisfy the precondition of the $\endaction$ of problem $\pddlProblem'$, which is the conjunction of the preconditions $p_1, \dots, p_k$ discovered during constraint compilation.
We work as in case 1; we assume that $s'_n\not\models p_i$, $1\leq i\leq k$, and discover that this implies that a constraint of problem $\pddlProblem$ is violated at state $s_n$, contradicting the assumption that $\plan$ is a solution for $\pddlProblem$.

Case 3 expresses that the final state $s'_f$ of plan $\plan'$ does not satisfy the goal $G'$ of problem $\pddlProblem'$.
Based on Algorithm \ref{alg:lifted_no_regr}, $G'$ is strictly more restrictive that goal $G$, as it additionally requires that all the $hold_c$ monitoring atoms and the $end$ atom are true.
The $end$ atom is true in $s'_f$ because it is an effect of action $\endaction$.
Moreover, we have $s_n\models G$, because $\plan$ is a solution for $\pddlProblem$, which implies that $s'_n\models G$, as $G$ does not contain monitoring atoms.
In turn $s'_n \models G$ implies that $s'_f\models G$ because the $\endaction$ does not affect the atoms in $G$.
As a result, $s'_f$ must be violating a $hold_c$ atom.
Thus, there is a constraint $c\val \ST(\phi)$ or $c\val \SA(\phi, \psi)$ in $\pddlProblem$ such that $s'_f\not\models hold_c$.
\begin{compactitem}
    \item Case $\ST(\phi)$. Since every action in $A'$ has effect $\cond{\phi}{hold_c}$ and $s'_f\not\models hold_c$, then there is no state in $\sigma'$ where $\phi$ holds. 
    This implies that neither $\sigma$ contains such a state, which is a contradiction because $\plan$ solves $\pddlProblem$ and $\ST(\phi)$ is one of its constraints.
    \item Case $\SA(\phi, \psi)$. Based on the effects manipulating $hold_c$ and the fact that $hold_c$ is true initially, we $s'_f\not\models hold_c$ iff there is a state where $\phi\wedge \neg\psi$ holds and it is not followed by a state where $\psi$ holds.
    Moving this situation to the trace $\sigma$ induced by plan $\plan$, we deduce that this plan violates constraint $\SA(\phi, \psi)$ of $\pddlProblem$, which is a contradiction because $\plan$ is a solution of $\pddlProblem$.
\end{compactitem}

By contradicting cases 1--3, we have shown that if $\plan$ is solution for $\pddlProblem$ then plan $\plan'$ is a solution for $\pddlProblem'$.

Suppose that $\plan'$ is a solution for $\pddlProblem'$, but $\plan$ is not solution for $\pddlProblem$.
One of the following must hold:
\begin{compactenum}
\item $\exists \actioninst_i, 0\leq i < n: s_i\not\models \actionPrec{\actioninst_i}$
\item $s_n\not\models G$
\item $\exists c\in C$ such that $\sigma$ does not satisfy $c$.
\end{compactenum}
Case 1 cannot be true as $\forall \actioninst_i, 0\leq i < n: s'_i\models Pre'(\actioninst_i)$, since $\plan'$ is a solution for $\pddlProblem'$, $Pre'(\actioninst_i)\models \actionPrec{\actioninst_i}$, and $s'_i\models \actionPrec{\actioninst_i}$ iff $s_i\models\actionPrec{\actioninst_i}$.
Case 2 cannot be true because we have $s'_f\models G'$, which is implies $s_n\models G$.

We treat case 3 depending on the type of constraint $c$.
\begin{compactitem}
    \item Case $\A(\phi)$. There is a state $s_i$ such that $s_i\models \neg\phi$, which implies that $s'_i\models \neg\phi$.
    However, since precondition $Pre'(\actioninst_i)$ requires that $\phi$ holds, that makes plan $\plan'$ not executable, which is a contradiction.
    \item Case $\ST(\phi)$. There is no state in $\sigma$ where $\phi$ holds. As a result, there is neither such a state in $\sigma'$, meaning that effect $\cond{\phi}{hold_c}$ is never triggered.
    Thus, $hold_c$ does not hold at $s'_f$, which is a contradiction as $hold_c$ is required in the goal of $\pddlProblem'$.
    \item Case $\AO(\phi)$. There are states $s_l$, $s_j$ and $s_i$ such that $l<j<i$, $\phi$ is true in $s_l$, $\phi$ is false in $s_j$ and $\phi$ is true in $s_i$.
    This implies that $\phi$ has the corresponding truth values in states $s'_l$, $s'_j$ and $s'_i$.
    Based on the action effects in $\pddlProblem'$, $seen_\phi$ holds after state $s'_l$ and $prevent_\phi$ holds after state $s'_j$.
    As a result, both $\phi$ and $prevent_\phi$ hold in state $s'_i$, making action $a_i$ inapplicable, which is a contradiction.
    \item Case $\SB(\phi,\psi)$. There is a state $s_i$ where $\phi$ holds and there is no state $s_j$, $j<i$, where $\psi$ holds.
    This implies that $s'_i\models \phi$ and $\forall j<i: s'_j\models \neg \psi$.
    As a result, action effect $\cond{\psi}{seen_\psi}$ is never satisfied before $s'_i$, meaning that $seen_\psi$ is false in $s'_i$.
    Therefore, based on precondition $\phi\rightarrow seen_\psi$, action $a_i$ is not applicable in state $s'_i$, which is a contradiction.
    \item Case $\SA(\phi,\psi)$. There is a state $s_j$ where $\phi\wedge \neg\psi$ holds and there is no state after $s_i$, $i>j$, where $\psi$ holds.
    This implies that $s'_j\models \phi\wedge\neg\psi$ and $\forall j<i\leq n \vee i=f: s'_i\models \neg\psi$.
    Therefore, based on action effects $\set{\cond{\phi \wedge \neg \psi}{\neg hold_{c}}}$ and $\{\cond{\psi}{hold_c}\}$, $hold_c$ does not hold in the final state $s'_f$, which is a contradiction because $hold_c$ is required by goal $G'$.
\end{compactitem}
By contradiction, we have shown that if plan $\plan'$ is a solution for $\pddlProblem'$ then $\plan$ is solution for $\pddlProblem$.

Therefore, if \lcc compiles a problem $\pddlProblem$ into problem $\pddlProblem'$, then plan $\tup{\actioninst_0, \dots, \actioninst_{n\minus 1}}$ is valid for $\pddlProblem$ iff plan $\tup{\actioninst_0, \dots, \actioninst_{n\minus 1}, \endaction}$ is valid for $\pddlProblem'$.
\end{proof}

\subsection{Proof of Proposition \ref{prop:lcc_complexity}} \label{proof:prop:lcc_complexity}

\proplcccomplexity*

\begin{proof}
In order to compile away the constraints $C$ in problem $\pddlProblem$, \lcc first needs to update the set of atoms $F$ of $\pddlProblem$ with monitoring atoms (see line \ref{line:lcc_fprime} of Algorithm \ref{alg:lifted_no_regr}).
This is achieved in one pass over the constraints in $C$, and thus the cost of this operation is $\mathcal{O}(n_c)$.
Second, \lcc iterates once over each constraint $c\in C$ in order to identify the set of preconditions $P$ and the set of effects $E$ that need to be added to the actions of $\pddlProblem$ (lines \ref{line:lcc_forconstr}--\ref{line:lcc_sa_eff}).
In each iteration, \lcc adds a fixed number of preconditions and effects in sets $P$ and $E$, and thus the cost of each iteration is $\mathcal{O}(1)$.
Therefore, the cost of computing sets $P$ and $E$ is $\mathcal{O}(n_c)$.
Third, \lcc adds preconditions $P$ and effects $E$ to every action in $\pddlProblem$, as well as to action $\endaction$, which requires $n_a\plus 1$ operations (lines \ref{line:lcc_foraction}--\ref{line:lcc_addendaction}).
Fourth, \lcc updates the initial state and the goal with $hold_c$ atoms by performing one iteration over the constraints in $C$ (line \ref{line:lcc_return}).

Based on the above steps, the worst-case time complexity of \lcc is $\mathcal{O}(3n_c \plus n_a\plus 1)$, which, after simplifications, becomes $\mathcal{O}(n_a\plus n_c)$.
\end{proof}
\section{Datasets}
\label{sec:appendix:experiments:constraints}

\subsection{Constrained Planning Problem Generator} \label{sec:appendix:experiments:constraints:generator}
\begin{figure}[t]
    \centering
    \includegraphics[width=.75\linewidth]{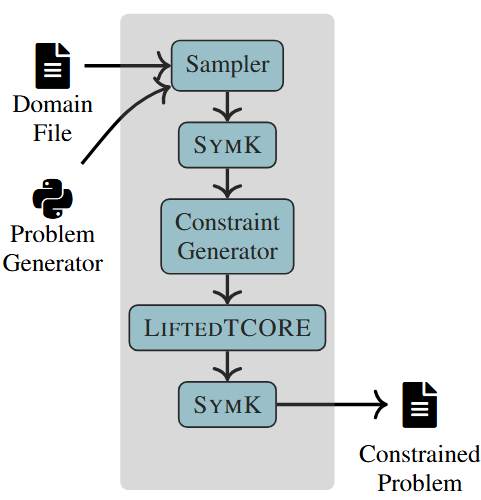}
    \caption{Constrained planning problem generator.}
    \label{fig:generator}
\end{figure}

We used a custom generator to introduce constraints to planning problems from the domains used in the latest International Planning Competition (IPC).
Figure \ref{fig:generator} presents the architecture of our constrained planning problem generator.
We provide as input to our generator a PDDL domain file and a problem instance generator for one of the planning domains used in our analysis, both of which have been made public by the latest IPC.
The output of our generator is a PDDL planning problem with constraints for the provided domain.

Our generator operates as follows.
First, we sample a problem file using the input problem generator (see `Sampler' in Figure \ref{fig:generator}); this problem file does not include constraints.
Second, we invoke the off-the-shelf planner \symk\ to derive an optimal plan for the sampled problem.
Third, we generate constraints for the sampled problem (see `Constraint Generator').
Constraint generation is task-aware, in the sense that we guide the search towards constraints that complicate the optimal solution of the unconstrained problem discovered by \symk\ in the previous step, while preserving feasibility, i.e., the constrained problem has a solution. 
We add the produced constraints to the sampled problem file, leading to planning problem with constraints.
Fourth, we use \lccr\ to compile away the constraints from this problem and, fifth, we invoke \symk\ to identify an optimal plan for the constrained planning problem.
If this plan is longer than the optimal plan of the original (unconstrained) problem, then we add the generated constrained planning problem to the dataset of our experimental evaluation.

\begin{figure}[t]
    \centering
    \includegraphics[width=.5\linewidth]{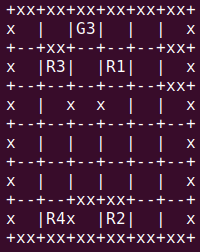}
    \caption{Initial state of a problem from the `Ricochet Robots' domain. `R1', `R2', `R3' and `R4' denote robots, `G3' is a goal location, and `X' symbols denote walls.}
    \label{fig:ricochet_example}
\end{figure}

In order to construct constraints with quantifiers, we instruct our generator to modify the constraints it produced by replacing a subset of their constants with variables and introducing quantifiers referring to the newly added variables. 

We illustrate the constraint generation step of our generator with the following example. 
\begin{example}
Consider the `Ricochet Robots' planning domain, which comprises a 2D gridworld with walls blocking movement between certain squares.
A square on the board may be empty, contain a robot or a goal location. 
The task is to move each robot to its designated goal location. 
Robots can be moved by pushing them towards a cardinal direction; after being pushed, a robot continues moving until it reaches a wall or another robot.

Figure \ref{fig:ricochet_example} depicts the initial state of a problem from the `Ricochet Robots' domain.
The goal is to bring robot `R3' in goal location `G3'.
The optimal plan to solve this problem consists of two `push' actions.
The first one pushes robot R3 east, making it move one step until it hits robot R1.
The second push action moves R3 north, making it stop at goal location G3, as it hits the northern border of the gridworld. 

Our constraint generator analyses the trajectory induced by the above optimal plan of the problem and identifies a set of atoms whose inclusion in a constraint may complicate this planning problem.
In our example, executing the optimal plan yields a state where robot R3 is in location `L23', i.e., the square below G3.
Our constraint generator recognises that, if we were to disallow robot R3 from stepping on square L23, then the previously discovered optimal plan would not be executable, necessitating a lengthier plan to solve the problem.
Thus, our constraint generator may produce constraint $\A(\neg(\atF(\rthree,\ltwothree)))$, stating that robot R3 must never be on square L23.
As a result of adding this constraint to the problem, the length of its optimal solution becomes 4 moves: (i) moving R3 west, (ii) moving R3 north, (iii) moving R1 north and (iv) moving R3 west, after which R3 hits R1, thus stopping at goal location G3.

Our constraint generator may also complicate the problem by introducing constraints imposing intermediate goals.
For instance, in our example problem, we may generate constraint $\ST(\atF(\rthree,\lthreetwo))$, stating that, at some point during plan execution, robot R3 needs to be in location L32, i.e., the square that is below the starting position of R3.
Moreover, our generator may modify sampled constraints by introducing quantifiers.
Constraint $\A(\neg(\atF(\rthree,\ltwothree)))$, e.g., may be transformed as $\A(\forall R: \neg(\atF(R,\ltwothree)))$, expressing that no robot is ever allowed to step on square L23, while constraint $\ST(\atF(\rthree,\lthreetwo))$ may be transformed as $\ST(\exists R: \atF(R,\lthreetwo))$, expressing that, at some point during plan execution, there is a robot that is located on square L32.
\qedex
\end{example}

\begin{table}
\centering
\begin{tabular}{lcccccccccc} 
\centering
\multirow{2}{*}{} & \multicolumn{5}{c}{Ground Constraints}  \\
& \A & \ST & \AO & \SB & \SA \\
\cmidrule(r){2-6}
Folding & 2 & 4 & 1 & 8 & 5  \\
Labyrinth & 6 & 5 & 0 & 6 & 3 \\
Quantum & 4 & 4 & 4 & 4 & 4 \\
Recharging & 3 & 7 & 0 & 5 & 5  \\
Ricochet & 1 & 11 & 1 & 3 & 4  \\
Rubik's Cube & 2 & 7 & 0 & 4 & 7  \\
Slitherlink & 2 & 2 & 0 & 8 & 8  \\
\cmidrule(r){2-6}
Total & 20 & 40 & 6 & 38 & 36  \\
\end{tabular}
    \caption{Constraint type distribution in our dataset with ground constraints.}
    \label{tbl:constraint_types:ground}
\end{table}
\begin{table}
\centering
\begin{tabular}{lcccccccccc} 
\centering
\multirow{2}{*}{}  & \multicolumn{5}{c}{Non-Ground Constraints} \\
&  \A & \ST & \AO & \SB & \SA\\
\cmidrule(r){2-6}
Folding  & 2 & 3 & 3 & 9 & 3 \\
Labyrinth  & 6 & 5 & 0 & 6 & 3 \\
Quantum  & 1 & 4 & 1 & 7 & 7 \\
Recharging  & 3 & 7 & 0 & 5 & 5 \\
Ricochet  & 4 & 4 & 4 & 4 & 4 \\
Rubik's Cube  & 4 & 4 & 4 & 4 & 4 \\
Slitherlink  & 1 & 4 & 3 & 6 & 6 \\
\cmidrule(r){2-6}
Total  & 21 & 31 & 15 & 41  & 32 \\
\end{tabular}
    \caption{Constraint type distribution in our dataset with non-ground constraints.}
    \label{tbl:constraint_types:nonground}
\end{table}

\subsection{Constraint Type Statistics} \label{sec:appendix:experiments:constraints:distribution}

We report the constraint type distribution in our dataset.
Tables \ref{tbl:constraint_types:ground} and \ref{tbl:constraint_types:nonground} present the number of constraints of each type, i.e., $\A$, $\ST$, $\AO$, $\SB$ or $\SA$, per IPC domain, in our dataset with ground constraints, and in out dataset with non-ground constraints, respectively.
These constraints were produced by our constraints generator (Appendix \ref{sec:appendix:experiments:constraints:generator}), given domain files and unconstrained problem instances for the IPC domains. 
We observe that our generator more often complicated the unconstrained problem by introducing constraints that impose intermediate goals, i.e., $\ST$, $\SB$ and $\SA$.
On the other hand, it was often more difficult to find $\A$ and $\AO$ constraints that complicate the problem while maintaining solvability, resulting in relatively fewer $\A$ and $\AO$ constraints.

\section{Coverage on ICAPS 2021 Datasets}
\label{sec:appendix:experiments:icaps21}

Table~\ref{tab:icaps21} reports the coverage of all evaluated systems on the benchmark introduced by~\citet{DBLP:conf/aips/BonassiGPS21}. We excluded the Openstack domain, as it is fully grounded and falls outside the scope of this work. Both \lcc and \lccr perform poorly on this dataset, which contains many constraints that, when compiled, lead to large preconditions and effects. As a result, LAMA often fails during preprocessing, frequently running out of memory on instances compiled by lifted approaches. We were also unable to run experiments on the Storage domain due to incompatibilities with the Unified Planning library used to implement \lcc and \lccr. Among the lifted compilations, LTL-C works well, and this is mainly because it introduces PDDL axioms, which help decompose formulas and reduce the complexity of compiled actions. TCORE performs particularly well on this dataset, as working with the ground representation allows for more targeted, sparse updates to actions.

\begin{table}
    \centering
    \begin{tabular}{llcccc}
    \toprule
    &Domain & \lcc & LiftedTCORE & TCORE & LTL-C\\
    \midrule
    \multirow{4}{*}{\rotatebox{90}{ICAPS21}}&Rovers & 25 & 39 & \textbf{92} & 40 \\
    &Trucks & 16 & 0 & \textbf{78} & 39 \\
    &Storage & N.A. & N.A. & \textbf{33} & 24 \\
    &Tpp & 8 & 0 & 24 & \textbf{58} \\
    \midrule
    &Total & 49 & 39 & \textbf{227} & 161 \\
    \end{tabular}
    \caption{Coverage of all systems across all domains of the ICAPS21 benchmark. N.A. Indicates not applicable.}
    \label{tab:icaps21}
\end{table}

\end{document}